\newif\ifdrafting
\draftingfalse 

\RequirePackage{fix-cm}

\documentclass[twocolumn]{svjour3}          
\smartqed  

\usepackage{url}
\usepackage{cite}
\usepackage{times}
\usepackage{epsfig}
\usepackage{epstopdf}
\usepackage{graphicx}
\usepackage{amsmath}
\usepackage{amssymb}
\usepackage{color}
\usepackage{tikz}
\usetikzlibrary{positioning}
\usetikzlibrary{decorations.pathreplacing}
\usepackage{subcaption}
\usepackage{natbib}

\definecolor{darkgreen}{RGB}{0,127,0}
\definecolor{darkblue}{RGB}{0,0,127}
\definecolor{darkred}{RGB}{127,0,0}
\definecolor{darkmagenta}{RGB}{127,0,127}
\definecolor{darkcyan}{RGB}{0,127,127}

\allowdisplaybreaks

\ifdrafting
\newcommand{\AM} [1] {\textcolor{darkblue}{[AM: #1]}} 
\newcommand{\YM}[1] {\textcolor{darkgreen}{[YM: #1]}} 
\newcommand{\Todo} [1] {\textcolor{darkmagenta}{\bf [Todo: #1]}}
\else
\newcommand{\AM} [1] {}
\newcommand{\YM} [1] {}
\newcommand{\Todo} [1] {}
\fi

\DeclareMathSymbol{\Gamma}{\mathalpha}{operators}{0}

\begin{document}

\title{Statistics of the Distance Traveled until Connectivity\\ for Unmanned Vehicles
\thanks{A small part of this paper has appeared in \cite{muralidharan2017first}, in which we characterized the statistics of the distance traveled until connectivity, only for straight-line paths.
In contrast, in this paper, we have extensively extended this analysis and developed new mathematical tools to tackle more general paths.
Moreover, for the case where we consider multipath in Section \ref{sec:with_multipath}, we have proposed a new methodology to compute the statistics, which is significantly more efficient.
Finally, on the numerical results side, we have extensive validation of our theoretical analysis with numerical results, based on real channel parameters.
This work is supported in part by NSF CCSS award 1611254.}
}

\author{Arjun Muralidharan         \and
        Yasamin Mostofi
}

\institute{Arjun Muralidharan \at
              \email{arjunm@ece.ucsb.edu}
           \and
           Y. Mostofi \at
              \email{ymostofi@ece.ucsb.edu}
            \and              
           Department of Electrical and Computer Engineering,\\
              University of California Santa Barbara,\\
              Santa Barbara, CA 93106, USA
}

\date{Received: date / Accepted: date}

\maketitle

\begin{abstract}
In this paper, we consider a scenario where a robot needs to establish connectivity with a remote operator or another robot, as it moves along a path.
We are interested in answering the following question: what is the distance traveled by the robot along the path before it finds a connected spot?
More specifically, we are interested in characterizing the statistics of the distance traveled along the path before it gets connected, in realistic channel environments experiencing path loss, shadowing and multipath effects.
We develop an exact mathematical analysis of these statistics for straight-line paths and also mathematically characterize a more general space of paths (beyond straight paths) for which the analysis holds, based on the properties of the path such as its curvature.
Finally, we confirm our theoretical analysis using extensive numerical results with real channel parameters from downtown San Francisco.

\end{abstract}

\section{Introduction}\label{sec:intro}
There has been considerable research on a team of unmanned vehicles carrying out a wide range of tasks such as search and rescue, surveillance, agriculture, and  environment monitoring (\cite{YanMostofiTCSN14, tokekar2016sensor}).
Communication between such a team of robots and a remote operator or within the robotic network itself, is often crucial for the successful completion of these tasks.
For instance, consider a scenario where a robot has collected information about its environment and needs to transmit this information to a remote operator or another robot.
In order to do so, it first needs to establish a connection with the remote operator or the other robot.
The robot may not be able to do so at its current location and may need to move to establish a connection, exploiting the spatial variations of the channel quality.
This paper then answers the following question: what are the statistics of the distance traveled along a given path until connectivity?

There has been considerable recent interest in the area of connectivity in robotic systems.
For instance, in \cite{zavlanos2011graph}, the connectivity of a network is maximized using a graph-theoretic analysis while in \cite{yan2012robotic}, connectivity is optimized using a more realistic channel model.
There has also been work on path planning to enable connectivity (\cite{muralidharan2017path, sergio2017RCAMP, zeng2017energy, chatzipanagiotis2016distributed, muralidharan2017energy, yan2012robotic}) as well as on communication-aware sensing (\cite{YanMostofiTCSN14}).

However, a mathematical characterization of the statistics of the distance traveled until connectivity is lacking in the literature, which is the main motivation for this paper.
We refer to this problem as \emph{the first passage distance (FPD) problem}, analogous to the concept of first passage time \linebreak (\cite{siegert1951first}).
We next summarize the contributions of the paper.

\textbf{Statement of contributions:}
We mathematically characterize the probability density function (PDF) of the FPD as a function of the underlying channel parameters of the environment, such as shadowing, path loss, and multipath fading parameters.
We do so for two cases: 1) when ignoring the multipath component (which could be of interest when the robot looks for an area of good connectivity as opposed to a single spot, or when multipath is negligible), and
2) when considering the multipath component.
In both cases, we first develop an exact characterization of the statistics of the FPD for the setting with straight paths.
We utilize tools from the stochastic equation literature to characterize the FPD while ignoring the multipath component,
and develop a recursive characterization for the case when we include multipath.
We then mathematically characterize a more general space of paths for which the analysis holds, based on properties of the path such as its curvature.

Note that the PDF of the FPD can be directly computed via a high dimensional integration, as we will discuss in Section \ref{subsec:straight_path_with_mp}.
However, this direct computation is infeasible for moderate distances.
Our proposed theoretical framework is not only computationally efficient but 
also brings a foundational analytical understanding to the FPD and can significantly affect networked robotic operation design.

The paper is organized as follows.
In Section \ref{sec:problem_setup}, we formally introduce the problem and briefly summarize the channel's underlying dynamics.
In Section \ref{sec:without_multipath}, we characterize the statistics of the distance traveled until connectivity while ignoring the multipath component.
In Section \ref{sec:with_multipath}, we characterize the statistics of the FPD while including the effect of multipath in the analysis.
Finally, in Section \ref{sec:numerical_results}, we validate our mathematical characterizations through extensive simulation with real channel parameters from downtown San Francisco.

\section{Problem Setup}\label{sec:problem_setup}
Consider a robot traveling along a given trajectory that needs to get connected to either a remote operator or another robot, as shown in Fig. \ref{fig:setup_general_straight}a.
In order for the robot to successfully connect with the remote operator, the receptions need to satisfy a Quality of Service (QoS) requirement such as a target Bit Error Rate, which in turn results in a minimum required received Signal to Noise Ratio, or equivalently a minimum required channel power, given a fixed transmission power.
We denote this minimum required received channel power as $\gamma_{\text{th}}$ in this paper.
This paper then asks the following question: \textbf{\emph{What is the distance traveled by the robot along the path before it gets connected to the remote operator?}}
More specifically, we are interested in mathematically characterizing the probability density function (PDF) of this distance, for a given path, as a function of the underlying channel parameters, such as path loss, shadowing and multipath fading parameters, as well as the parameters of the path, such as its curvature.

\begin{figure}
    \centering
    \includegraphics[width=0.95\linewidth]{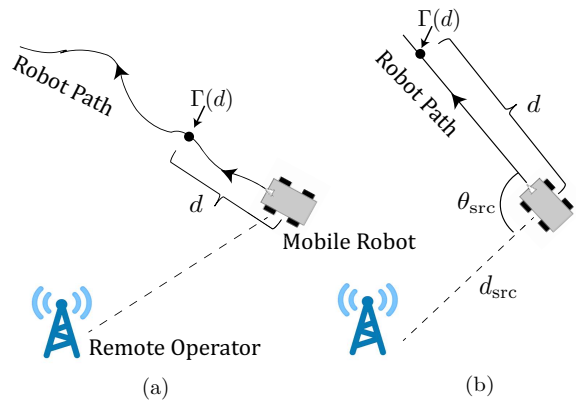}
    \caption{\small An example of the considered scenario for (a) a general path and (b) a straight path.}
    \label{fig:setup_general_straight}
\end{figure}

\subsection{Channel Model}\label{subsec:channel_model}

In the communication literature, channel power is well modeled as a multi-scale random process with three major dynamics: path loss, shadowing and multipath fading (\cite{rappaport1996wireless}).
Let $\Gamma(q)$ represent the received channel power (in the dB domain) at location $q \in \mathbb{R}^{2}$ with the remote operator located at the origin.
$\Gamma(q)$ can then be expressed as $\Gamma(q) = \gamma_{\text{PL}}(q) + \Gamma_{\text{SH}}(q) + \Gamma_{\text{MP}}(q)$ where $\gamma_{\text{PL}}(q) = K_{\text{dB}} - 10n_{\text{PL}}\log_{10}\|q\|$ is the distance-dependent path loss with $n_{\text{PL}}$ representing the path loss exponent, and $\Gamma_{\text{SH}}$ and $\Gamma_{\text{MP}}$ are random variables denoting the impact of shadowing and multipath respectively (in dB).
The multipath component or small-scale fading represents fluctuations in the channel power in the order of a wavelength, while the shadowing component or large-scale fading represents fluctuations of the channel power after the signal is locally averaged over multipath, thus reflecting the impact of larger objects such as blocking buildings.
$\Gamma_{\text{SH}}(q)$ is best modeled as a Gaussian random process with an exponential spatial correlation, i.e.,
$\mathbb{E}\left\{\Gamma_{\text{SH}}(q_1)\Gamma_{\text{SH}}(q_2)\right\} =  \sigma_{\text{SH}}^2 e^{-\|q_1-q_2\|/\beta_{\text{SH}}}$ where $\sigma_{\text{SH}}^2$ is the shadowing power and $\beta_{\text{SH}}$ is the decorrelation distance \linebreak(\cite{rappaport1996wireless}).
As for multipath, a number of distributions such as Nakagami, Rician and lognormal have been found to be a good fit (in the linear domain) (\cite{rappaport1996wireless, hashemi1994study}).

Consider the case where the robot is traveling along a path.
Let $d$ be the distance traveled by the robot along this path.
With a slight abuse of notation, in the rest of the paper we let $\Gamma(d)$ represent the channel power when the robot has traveled distance $d$ along the path, as marked in Fig. \ref{fig:setup_general_straight}a.
We thus have $\Gamma(d) = \gamma_{\text{PL}}(d) + \Gamma_{\text{SH}}(d) + \Gamma_{\text{MP}}(d)$.

\section{Characterizing the FPD Without Considering Multipath}\label{sec:without_multipath}

We start our analysis by ignoring the multipath and only considering the shadowing and path loss components of the channel, i.e., we want $\Gamma(d) = \gamma_{\text{PL}}(d) + \Gamma_{\text{SH}}(d)$ to be above $\gamma_{\text{th}}$.
This assumption allows us to better analyze and understand the FPD, and paves the way towards our most general characterization of the next section, which includes multipath as well.
Moreover, the analysis also has practical values of its own, and would be relevant to the case where the robot is interested in finding a general area of good connectivity as opposed to a single good spot.
In this section, we will characterize the statistics of the distance traveled until connectivity for this scenario.
We begin by analyzing straight paths in Section \ref{subsec:straight_path_without_mp}, where we utilize the stochastic differential equation literature (\cite{gardiner2009stochastic}) in our characterization.
We then extend our analysis to a more general space of paths in Section \ref{subsec:approx_markov_paths_without_mp}.

\subsection{Straight Paths: Stochastic Differential Equation Analysis}\label{subsec:straight_path_without_mp}

In this section, we characterize the PDF of the distance traveled until connectivity for straight-line paths.
Consider a robot situated at a distance $d_{\text{src}}$ from a remote operator or from another robot to which it needs to be connected, and moving in the direction specified by the angle $\theta_{\text{src}}$, as shown in Fig. \ref{fig:setup_general_straight}b.
The angle $\theta_{\text{src}}$ is measured clockwise with respect to the line segment connecting the remote operator and the robot, as can be seen in Fig. \ref{fig:setup_general_straight}b, and denotes the direction of travel chosen by the robot.

$\Gamma(d)$ represents the channel power when the robot is at distance $d$ along direction $\theta_{\text{src}}$, as marked in Fig. \ref{fig:setup_general_straight}b.
We thus have $\Gamma(d) = \gamma_{\text{PL}}(d) + \Gamma_{\text{SH}}(d)$, where
\begin{align}\label{eq:path_loss_comp}
\gamma_{\text{PL}}(d) = K_{\text{dB}} - 5n_{\text{PL}}\log_{10}(d_{\text{src}}^2 + d^2 -  2d_{\text{src}}d\cos\theta_{\text{src}}),
\end{align}
and $\Gamma_{\text{SH}}(d)$ is a zero mean Gaussian process with the spatial correlation of $\mathbb{E}\left\{\Gamma_{\text{SH}}(l)\Gamma_{\text{SH}}(d)\right\} = \sigma_{\text{SH}}^2 e^{-(d-l)/\beta_{\text{SH}}}$, \linebreak with $d \geq l$.
Note that $\Gamma(d)$ is also a function of $d_{\text{src}}$ and $\theta_{\text{src}}$.
We drop $\Gamma(d)$'s dependency on them in the notation as the analysis of the paper is carried out for a fixed $d_{\text{src}}$ and $\theta_{\text{src}}$.

As we shall see, $\Gamma_{\text{SH}}(d)$ becomes an Ornstein-Uhlenbeck process, one of the most studied types of Gauss-Markov processes (\cite{ricciardi1988first}; \cite{ricciardi1979ornstein}, \cite{leblanc1998path, gardiner2009stochastic}).
\linebreak Ornstein-Uhlenbeck process appears in many practical scenarios, such as Brownian motion, financial stock markets, or neuronal firing (\cite{ricciardi1979ornstein}; \linebreak \cite{leblanc1998path}), and thus has been heavily studied in the literature.
In this paper, we shall utilize this rich literature (\cite{gardiner2009stochastic}; \cite{di2001computational}) to mathematically characterize the FPD to connectivity for a mobile robot.

We begin by summarizing the definitions of a Gaussian process and a Markov process.
\begin{definition}[Gaussian Process] (\cite{dudley2002real})
A stochastic process $\{X(t): t \in T\}$, where $T$ is an index set, is a Gaussian process, if any finite number of samples have a joint Gaussian distribution, i.e., $(X(t_1), X(t_2), \cdots, X(t_k))$ is a Gaussian random vector for all $t_1,\cdots,t_k \in T$ and for all $k$.
\end{definition}
A Gaussian process is completely specified by its mean function $\mu(t) = \mathbb{E}[X(t)]$ and its covariance function $C(s,t) = \mathbb{E}\left\{[X(s)-\mu(s)][X(t)-\mu(t)]\right\}$.
We use the notation $X\sim \mathcal{GP}\left(\mu,C\right)$ to denote the underlying process.

\begin{definition}[Markov Process](\cite{papoulis2002probability})
A process $X(t)$ is Markov if 
\begin{align*}
&\mathrm{Pr}\left(X(t_n)\leq x_n|X(t_{n-1}),\cdots,X(t_1)\right) = \\
&\;\;\;\;\;\;\;\;\;\;\;\;\;\;\;\;\;\;\;\;\;\;\;\;\;\;\;\;\;\;\;\;\;\;\;\;\;\;\;\;\;\;\;\;\;\;  \mathrm{Pr}\left(X(t_n)\leq x_n|X(t_{n-1})\right),
\end{align*}
for all $n$ and for all $t_n \geq t_{n-1} \geq \cdots \geq t_1$, where $\mathrm{Pr}(.)$ denotes the probability of the argument.
\end{definition}

\begin{definition}[Gauss-Markov Process](\cite{mehr1965certain})
A stochastic process is Gauss-Markov if it satisfies the requirements of both a Gaussian process and a Markov process.
\end{definition}

We next state a lemma that shows when a Gaussian process is also Markov, which we shall utilize to prove that the channel shadowing power $\Gamma_{\text{SH}}(d)$ is Gauss-Markov.
\begin{lemma}\label{lemma:gauss_markov_cond}
A Gaussian process $X \sim \mathcal{GP}(\mu,C)$ is Markov if and only if
$C(s,u) = C(s,t)C(t,u)/C(t,t)$,
for all $u\geq t \geq s$.
\end{lemma}
\begin{proof}
See \cite{doob1949heuristic} for the proof.
\end{proof}

\begin{corollary}
The channel shadowing power $\Gamma_{\text{SH}}(d)$ and the channel power $\Gamma(d)$ are Gauss-Markov processes.
\end{corollary}
\begin{proof}
$\Gamma_{\text{SH}} \sim \mathcal{GP}(0, C_{\Gamma_{\text{SH}}})$ is a Gaussian process with zero mean and covariance $C_{\Gamma_{\text{SH}}}(s,u)=  \sigma_{\text{SH}}^{2} e^{-(u-s)/\beta_{\text{SH}}}$.
This covariance function satisfies $C_{\Gamma_{\text{SH}}}(s,t)C_{\Gamma_{\text{SH}}}(t,u)/C_{\Gamma_{\text{SH}}}(t,t) = \sigma_{\text{SH}}^2e^{-(u-t)-(t-s)/\beta_{\text{SH}}}= C_{\Gamma_{\text{SH}}}(s,u)$, for $u \geq t \geq s$, which concludes the proof for $\Gamma_{\text{SH}}(d)$ using Lemma \ref{lemma:gauss_markov_cond}.
The channel power $\Gamma(d)$ is the sum of $\Gamma_{\text{SH}}(d)$ and a mean function (path loss function $\gamma_{\text{PL}}(d)$).
Thus, the channel power is also a Gauss-Markov process with distribution $\Gamma \sim \mathcal{GP}(\gamma_{\text{PL}}, C_{\Gamma_{\text{SH}}})$.
\end{proof}

\begin{remark}[see \cite{gardiner2009stochastic}]
The Ornstein-Uhlenbeck process $O \sim \mathcal{GP}(0,C_{O})$ is a Gauss-Markov process with the covariance function $C_{O}(s,u) = \sigma^{2} e^{-(u-s)/\beta}$, where $\sigma \geq 0$ and $\beta\geq0$ are constants.
Thus, we can see that $\Gamma_{\text{SH}}(d)$ is an Ornstein-Uhlenbeck process.
\end{remark}

In order to gain more insight into the stochastic process $\Gamma(d)$, we next discuss the transition PDF $f(\gamma, d|\eta, l) =  \frac{\partial}{\partial \gamma} \mathrm{Pr}\left(\Gamma(d) < \gamma|\Gamma(l) = \eta\right)$, where $d \geq l$, as well as the stochastic differential equation governing $\Gamma(d)$, both of which we shall subsequently use in our characterization of the PDF of the FPD.

\subsubsection{The Underlying Stochastic Differential Equation}
The transition PDF $f(\gamma, d|\eta, l)$ characterizes the distribution of $\Gamma(d)$ given $\Gamma(l) = \eta$.
This is a normal density characterized by a mean and variance of (see 10.5 of \cite{kay1993fundamentals})
\begin{align}\label{eq:mean_var_trans_pdf}
\mathbb{E}\left[\Gamma(d)|\Gamma(l) = \eta\right] &= \gamma_{\text{PL}}(d) + e^{-(d-l)/\beta_{\text{SH}}}(\eta - \gamma_{\text{PL}}(l)) \nonumber\\ 
\text{Var}\left[\Gamma(d)|\Gamma(l) = \eta\right] & = \sigma_{\text{SH}}^2(1 - e^{-2(d-l)/\beta_{\text{SH}}}). 
\end{align}
The transition PDF explicitly shows the spatial dependence of the channel power $\Gamma(d)$. 
As stated in \cite{di2001computational}, $f(\gamma, d|\eta, l)$ satisfies the partial differential equation known as the forward Fokker-Planck equation:\footnote{The Fokker-Planck equation of \cite{di2001computational} is stated for a general Gauss-Markov process. 
Here we adapted it for our specific Gauss-Markov process $\Gamma(d)$.} 
\begin{align}\label{eq:fokker_planck}
\frac{\partial}{\partial d}f(\gamma, d|\eta, l) &= -\frac{\partial}{\partial \gamma}\left[A(\gamma,d)f(\gamma, d|\eta, l)\right] \nonumber\\ 
& \;\;\;\;\;\;\;\;\;\;\;\;\;\;\;\;\; + \frac{1}{2}\frac{\partial^2}{\partial \gamma^2}\left[Bf(\gamma,d|\eta,l)\right], 
\end{align}
with the associated initial condition of $f(\gamma, l|\eta, l) = \delta(\gamma - \eta)$, 
where $A(\gamma,d) = \gamma_{\text{PL}}'(d) - \left(\gamma - \gamma_{\text{PL}}(d)\right)/\beta_{\text{SH}}$, 
$B = (2\sigma_{\text{SH}}^2)/\beta_{\text{SH}}$ and $\gamma_{\text{PL}}(d)$ is as stated in (\ref{eq:path_loss_comp}), with its derivative:
\begin{align*}
\gamma_{\text{PL}}'(d) = -10n_{\text{PL}}\log_{10}(e)\frac{d - d_{\text{src}}\cos\theta_{\text{src}}}{d_{\text{src}}^2 + d^2 -  2d_{\text{src}}d\cos\theta_{\text{src}}}.
\end{align*}
The Fokker-Planck equation shows the evolution of the probability density $f(\gamma, d|\eta, l)$ with the traveled distance $d$ given $\Gamma(l)=\eta$.

Moreover, as shown in \cite{gardiner2009stochastic}, the channel power $\Gamma(d)$ can be represented as a stochastic differential equation:\footnote{\cite{gardiner2009stochastic} provides the stochastic differential equation for the Ornstein-Uhlenbeck process, from which we can easily obtain (\ref{eq:sde_ou}).}
\begin{align}\label{eq:sde_ou}
\mathrm{d}\Gamma(d) & = A(\Gamma,d)\mathrm{d}d + \sqrt{B}\mathrm{d}W(d),
\end{align}
where $W(d)$ is the Wiener process and $A(\gamma, d)$ and $B$ are as defined before.
\begin{remark}\label{remark:drift_diffusion}
In (\ref{eq:fokker_planck}) and (\ref{eq:sde_ou}), $A(\gamma, d)$ and $B$ are known as the drift and the diffusion components respectively.
The drift $A(\gamma,d) = \gamma_{\text{PL}}'(d) - \left(\gamma - \gamma_{\text{PL}}(d)\right)/\beta_{\text{SH}}$ is a pull towards the mean, and the diffusion component $B = (2\sigma^{2}_{\text{SH}})/\beta_{\text{SH}}$ is a function of the shadowing variance and the decorrelation distance.
Then, in an increment $\Delta d$, we can think of the channel power spatially evolving  with a deterministic rate $A(\gamma,d)$, in addition to a random Gaussian term with the variance $B \Delta d$.
\end{remark}

Next, we utilize our established lemmas to derive the PDF of the FPD.

\subsubsection{First Passage Distance}
Consider the random variable $\mathcal{D}_{\gamma_0} = \inf_{d\geq 0}\{d:\Gamma(d) \geq \gamma_{\text{th}}|\Gamma(0)=\gamma_0<\gamma_{\text{th}}\}$.
This denotes the FPD of the process $\Gamma(d)$ to the connectivity threshold $\gamma_{\text{th}}$, with the initial value $\Gamma(0) = \gamma_0 < \gamma_{\text{th}}$.
Further, let
$g[d|\gamma_0] = \frac{\partial}{\partial d}\mathrm{Pr}\left(\mathcal{D}_{\gamma_0} < d\right)$
represent the PDF of the FPD.
In the following theorem, we characterize this PDF.

\begin{theorem}\label{theorem:fpd}
The PDF of FPD $g[d|\gamma_0]$ satisfies the following non-singular second-kind Volterra integral equation:
\begin{align}\label{eq:fpd_pdf}
g[d|\gamma_0] = -2\Psi[d|\gamma_0,0] + 2\int_{0}^{d}g[l|\gamma_0]\Psi[d|\gamma_{\text{th}},l]\mathrm{d}l,
\end{align}
where $\gamma_0 < \gamma_{\text{th}}$ and
\begin{align}\label{eq:psi_volterra}
\Psi[d|\eta, l] = \Bigg\{-\frac{1}{2}\frac{\mathrm{d}\gamma_{\text{PL}}(d)}{\mathrm{d}d} - \frac{\gamma_{\text{th}} - \gamma_{\text{PL}}(d)}{2\beta_{\text{SH}}} \frac{1+e^{-2(d-l)/\beta_{\text{SH}}}}{1-e^{-2(d-l)/\beta_{\text{SH}}}} \nonumber\\ 
+ \frac{\eta - \gamma_{\text{PL}}(l)}{\beta_{\text{SH}}} \frac{e^{-(d-l)/\beta_{\text{SH}}}}{1-e^{-2(d-l)/\beta_{\text{SH}}}} \Bigg\} f(\gamma_{\text{th}},d|\eta,l). 
\end{align}
\end{theorem}
\begin{proof}
The proof is based on the fact that $\Gamma(d)$ is a Gauss-Markov process and utilizes the Fokker-Planck equation (\ref{eq:fokker_planck}).
The details are then adapted from Theorem 3.1 of \cite{di2001computational} to our particular Gauss-Markov process.
\end{proof}

$\mathcal{D}_{\gamma_0}$ represents the FPD for a given initial value of $\Gamma(0) = \gamma_0$.
In many scenarios, we are instead interested in characterizing the FPD for the initial state $\Gamma(0)$ being a random variable bounded from above by $\gamma_{\text{th}}$, i.e., we are interested in characterizing the FPD when the starting position is not connected.
This is known as the upcrossing FPD in the general first passage literature (\cite{di2001computational}).
We next extend our analysis to derive the PDF of the upcrossing FPD.
Let the random variable $\mathcal{D}_{\Gamma_0}^{(\epsilon)} = \inf_{d\geq 0}\{d:\Gamma(d) \geq \gamma_{\text{th}}|\Gamma(0)<\gamma_{\text{th}}-\epsilon\}$ denote the $\epsilon$-upcrossing FPD of $\Gamma(d)$ to the boundary $\gamma_{\text{th}}$ given that the initial state satisfies $\Gamma(0) < \gamma_{\text{th}}-\epsilon$, where $\epsilon >0$ is a fixed real number.
The $\epsilon$-upcrossing FPD, $\mathcal{D}_{\Gamma_0}^{(\epsilon)}$, can be characterized as follows:
\begin{align*}
\mathrm{Pr}\left(\mathcal{D}_{\Gamma_0}^{(\epsilon)} < d\right) = \int_{-\infty}^{\gamma_{\text{th}}-\epsilon}\mathrm{Pr}\left(\mathcal{D}_{\gamma_0} < d\right)\zeta_{\epsilon}(\gamma_0)\mathrm{d}\gamma_0,
\end{align*}
where $\mathcal{D}_{\gamma_0}$ is the FPD given the initial value $\Gamma(0) = \gamma_0<\gamma_{\text{th}}$, as defined earlier, and 
\begin{align*}
\zeta_{\epsilon}(\gamma_0) = \left\{\begin{array}{ll}
\frac{f(\gamma_0, 0)}{\mathrm{Pr}(\Gamma(0) < \gamma_{\text{th}} - \epsilon)}, & \gamma_0 < \gamma_{\text{th}} - \epsilon \\
0, & \gamma_0 \geq \gamma_{\text{th}} - \epsilon
\end{array}\right.,
\end{align*}
is the PDF of $\Gamma(0)|\Gamma(0)<\gamma_{\text{th}}-\epsilon$ with $f(\gamma,d)$ denoting the PDF of $\Gamma(d)$.
Moreover, the $\epsilon$-upcrossing FPD density $g_{u}^{(\epsilon)}[d] = \frac{\partial}{\partial d} \mathrm{P}(\mathcal{D}_{\Gamma_0}^{(\epsilon)} < d)$ is similarly related to the FPD density $g[d|\gamma_0]$ as follows:
$g_{u}^{(\epsilon)}[d] = \int_{-\infty}^{\gamma_{\text{th}}-\epsilon}g[d|\gamma_0]\zeta_{\epsilon}(\gamma_0)\mathrm{d}\gamma_0$.

\begin{remark}
Note that we have required $\epsilon>0$.
This is due to the fact that the mathematical tools we shall utilize are not well-defined for $\gamma_0 = \gamma_{\text{th}}$.
However, $\epsilon$ can be chosen arbitrarily small.
\end{remark}

In the following theorem, we derive an expression for $g_{u}^{(\epsilon)}[d]$, the PDF of the $\epsilon$-upcrossing FPD.
\begin{theorem}\label{theorem:upcrossing_fpd}
The PDF of the $\epsilon$-upcrossing FPD, $g_{u}^{(\epsilon)}[d]$, satisfies the following non-singular second-kind Volterra integral equation:
\begin{align}\label{eq:ufpd_pdf}
g_{u}^{(\epsilon)}[d] = -2\Psi_{u}^{(\epsilon)}[d] + 2\int_{0}^{d}g_{u}^{(\epsilon)}[l]\Psi[d|\gamma_{\text{th}},l]\mathrm{d}l,
\end{align}
where $\Psi[d|\eta,l]$ is as defined in (\ref{eq:psi_volterra}),
\begin{align*}
\Psi_{u}^{(\epsilon)}[d] = \frac{1}{2\mathrm{Pr}(\Gamma(0)<\gamma_{\text{th}} - \epsilon)}\Bigg\{\frac{-2\sigma_{\text{SH}}^2}{\beta_{\text{SH}}}e^{-d/\beta_{\text{SH}}}f(\gamma_{\text{th}}-\epsilon, 0)\\
\times f[\gamma_{\text{th}},d|\gamma_{\text{th}}-\epsilon,0] + \frac{1}{2}f(\gamma_{\text{th}}, d)(1+\mathrm{Erf}[\Upsilon_{\epsilon}(d)])\\
\times \left(-\frac{\mathrm{d}\gamma_{\text{PL}}(d)}{\mathrm{d}d} - \frac{1}{\beta_{\text{SH}}}\left[\gamma_{\text{th}} - \gamma_{\text{PL}}(d)\right]\right)\Bigg\},
\end{align*}
with $\mathrm{Erf}(z) = \frac{2}{\sqrt{\pi}}\int_{0}^{z}e^{-t^2}\mathrm{d}t$ representing the error function, and
\begin{align*}
\Upsilon_{\epsilon}(d) = \frac{\gamma_{\text{th}} - \epsilon - \gamma_{\text{PL}}(0) - e^{-d/\beta_{\text{SH}}}\left(\gamma_{\text{th}} - \gamma_{\text{PL}}(d)\right)}{\sqrt{2\sigma_{\text{SH}}^2\left(1-e^{-2d/\beta_{\text{SH}}}\right)}}.
\end{align*}
\end{theorem}
\begin{proof}
The proof is obtained by adapting Theorem $5.3$ of \cite{di2001computational} to our particular Gauss-Markov process form. 
\end{proof}
In terms of implementation, the functions $\Psi[d|\eta,l]$ and $\Psi_{u}^{\epsilon}[d]$ in Theorem \ref{theorem:fpd} and Theorem \ref{theorem:upcrossing_fpd} can be easily computed.
The PDF of the FPD ($g[d|\gamma_0]$) and the PDF of the $\epsilon$-upcrossing FPD ($g_{u}^{(\epsilon)}[d]$) can then be computed from the integral equations (\ref{eq:fpd_pdf}) and (\ref{eq:ufpd_pdf}) respectively.
In particular, Simpson rule provides the basis for an efficient iterative algorithm for evaluating these integrals (See Section 4 of \cite{di2001computational}).

\begin{remark}[Computational complexity]
The direct computation of $g_{u}^{(\epsilon)}[d]$ involves a high dimension integration, as we will discuss in Section \ref{subsec:straight_path_with_mp}.
For a discretized path of $N$ steps, this direct computation would have a computational cost exponential in $N$, i.e. $O(NM^{N})$ for some constant $M$.
In contrast, the computation cost of $g_{u}^{(\epsilon)}[d]$ using Theorem \ref{theorem:upcrossing_fpd} is $O(N^2)$.
Moreover, Theorem \ref{theorem:upcrossing_fpd} is also an elegant characterization of the $\epsilon$-upcrossing FPD that can be utilized for analysis and design of robotic operations.
\end{remark}

\subsection{Approximately-Markovian Paths}\label{subsec:approx_markov_paths_without_mp}

In this section, we characterize the space of paths (beyond straight paths) that results in approximately-Markovian processes.
As we saw in Section \ref{subsec:straight_path_without_mp}, the channel shadowing component along a straight line is a Gauss-Markov process.
This allowed us to characterize the statistics of the distance to connectivity for a mobile robot traveling along a straight path.
A general non-straight path is not Markovian since the covariance function $C_{\Gamma_{\text{SH}}}(s,u)$ does not satisfy Lemma \ref{lemma:gauss_markov_cond}.
In this section, we characterize the space of paths for which the channel shadowing power along the path is approximately a Gauss-Markov process.
This allows us to immediately apply the stochastic differential equation analysis of Section \ref{subsec:straight_path_without_mp} to characterize the statistics of the distance until connectivity for these paths.

Consider the scenario in Fig. \ref{fig:path_rolling_ball_and_discretized} (top), where we have discretized the path, with $\Gamma_{\text{SH},-0}$ denoting the shadowing power at the current location and $\Gamma_{\text{SH}, -1}, \Gamma_{\text{SH},-2}, \cdots $ indicating the channel shadowing power at previously-visited points.\footnote{Note that the discretization step size of the path must be small for the derivations of Theorem \ref{theorem:upcrossing_fpd} to be valid.}
In Section \ref{subsec:approx_markov_paths_without_mp}, we saw that a Gauss-Markov process satisfies the Fokker-Planck equation of (3), which provides us with the result of Theorem \ref{theorem:upcrossing_fpd}.
The Fokker-Planck equation in turn requires the property that $p(\gamma_{\text{SH},-0}| \gamma_{\text{SH},-1}, \gamma_{\text{SH},-2}, \cdots) = p(\gamma_{\text{SH},-0}| \gamma_{\text{SH},-1}) $ for its derivation (through the Chapman-Kolmogorov equation (\cite{gardiner2009stochastic})).
Thus, we say a path is approximately-Markovian, if at every point on the path, we have that 
$p(\gamma_{\text{SH},-0}| \gamma_{\text{SH},-1}, \gamma_{\text{SH},-2}, \cdots)$ is close to $p(\gamma_{\text{SH},-0}| \gamma_{\text{SH},-1}) $.
We will characterize this closeness precisely in Section \ref{subsubsec:approx_markov_kl_divergence} using the Kullback-Leibler (KL) divergence metric.

Our key insight is that the approximate Markovian nature is related to the curvature of a path, which is a measure of how much the path curves, i.e., how much it deviates from a straight line.
For instance, a straight line has a curvature of $0$.
Thus, we would expect that paths with small enough curvature would result in approximately-Markovian processes.
We will precisely characterize what we mean by this in Section \ref{subsubsec:curvature_constraint}.

We first describe an outline of our approach for characterizing the space of approximately-Markovian paths.
At every point on the path, instead of checking for the conditional distribution given all the past points on the path, which is cumbersome, we consider all past points on the path within a certain distance of the current point, i.e., within a ball centered at the current point.
In other words, to check the approximately-Markovian property, we evaluate \linebreak $p(\gamma_{\text{SH},-0}| \gamma_{\text{SH},-1}, \gamma_{\text{SH},-2}, \cdots, \gamma_{\text{SH},-n})$ instead of \linebreak $p(\gamma_{\text{SH},-0}| \gamma_{\text{SH},-1}, \gamma_{\text{SH},-2}, \cdots)$.
Fig. \ref{fig:path_rolling_ball_and_discretized} (top) shows an illustration of this.
This makes sense since the shadowing component has an exponential correlation function.
Thus, if the radius of the ball is large enough, the points outside of the ball will have a negligible impact on the estimate at the center of the ball.
We will characterize this radius in Section \ref{subsubsec:ball_radius_constraint}.
Thus, our strategy is to roll a ball along the path, as shown in Fig. \ref{fig:path_rolling_ball_and_discretized} (bottom), and to check if the approximate Markovian property holds at each point along the path.
We then characterize two conditions that can ensure that a path will be approximately-Markovian.
The first is that, at any point on the path, if we travel backward along the path it should not loop either within the ball or such that it re-enters the ball.
We refer to such looping as $d_{\text{th}}$-looping ($d_{\text{th}}$ being the radius of the ball), and examples of this are shown in figures \ref{fig:d_th_loop_free_balls_and_ball_characterization}a and \ref{fig:d_th_loop_free_balls_and_ball_characterization}b.
Equivalently, a path is called $d_{\text{th}}$-loop-free if there is no $d_{\text{th}}$-looping.
The second condition is that the maximum curvature of the path should be smaller than a certain bound, which will be characterized later in Section \ref{subsubsec:curvature_constraint}.
If the $d_{\text{th}}$-loop-free condition is satisfied, then the only part of the path that lies within the ball would lie in the shaded region of Fig.  \ref{fig:d_th_loop_free_balls_and_ball_characterization}c, and if the maximum curvature of the path is small enough, then the path will be approximately-Markovian.
We will formulate this precisely in Section \ref{subsubsec:curvature_constraint}.

\begin{figure}
    \centering
    \includegraphics[width=0.95\linewidth]{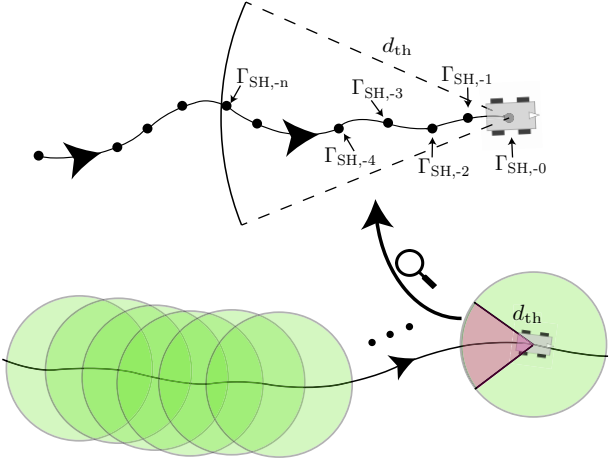}
    \caption{\small (bottom) A ball with radius $d_{\text{th}}$ rolling along the path, where we check for approximate Markovianity within each ball, and (top) the discretized path and the corresponding channel shadowing power values within a ball.}
    \label{fig:path_rolling_ball_and_discretized}
\end{figure}

We start by mathematically characterizing the $d_{\text{th}}$-looping condition in detail.

\subsubsection{$d_{\text{th}}$-Loop-Free Constraint}\label{subsubsec:loop_free_constraint}
We define $d_{\text{th}}$-loop-free paths as paths where neither of the two following scenarios occurs at any point on the path.
The first is when traveling backward along a path, the path loops within the ball itself.
More precisely, when traveling backward along the path, let the initial direction of travel be along the negative x-axis.
We say that the path loops within the ball if at any point (still inside the ball), the direction of travel has a component along the positive x-axis (e.g., Fig. \ref{fig:d_th_loop_free_balls_and_ball_characterization}a).
The second scenario is when the path re-enters the ball once it leaves it.
These two scenarios, which we collectively refer to as $d_{\text{th}}$-looping, are illustrated in figures \ref{fig:d_th_loop_free_balls_and_ball_characterization}a and \ref{fig:d_th_loop_free_balls_and_ball_characterization}b.
Such $d_{\text{th}}$-looping behavior can possibly invalidate the approximate Markovian nature of the path.

We next relate the $d_{\text{th}}$-loop-free condition to the curvature of the path.
We first review the precise definition of curvature.
\begin{definition}[Curvature](\cite{kline1998calculus})
The curvature of a planar path $r(s)=(x(s),y(s))$ parameterized by arc-length is defined as
\begin{align*}
\kappa(s) = \|T'(s)\|,
\end{align*}
where $T(s)$ is the unit tangent vector at $s$.
\end{definition}

When traveling backward along a path, consider the segment of the path inside the ball, before the path exits the ball.
Let $r_{\text{ball}}$ refer to this segment, as shown in Fig. \ref{fig:d_th_loop_free_balls_and_ball_characterization}c.
Moreover, let $d_{r_{\text{ball}}}$ refer to its length.
The following lemma characterizes some important properties of $r_{\text{ball}}$.
\begin{lemma}\label{lemma:ball_segment_length}
For a path with maximum curvature $\kappa$ and a ball with radius $d_{\text{th}}$, the path segment $r_{\text{ball}}$ satisfies the following properties:
\begin{enumerate}
\item $r_{\text{ball}}$ lies within the shaded region of Fig. \ref{fig:d_th_loop_free_balls_and_ball_characterization}c where the boundary of the region corresponds to circular arcs with curvature $\kappa$.
\item If  $\kappa < 1/d_{\text{th}}$, $r_{\text{ball}}$ cannot loop within the ball (see Fig. \ref{fig:d_th_loop_free_balls_and_ball_characterization}a for an example of looping within the ball).
\item The length of the segment $r_{\text{ball}}$ satisfies 
\begin{align*}
d_{r_
{\text{ball}}} < \frac{1}{\kappa}\sin^{-1}\left(\kappa \times d_{\text{th}}\right).
\end{align*}
\end{enumerate}
\end{lemma}

\begin{proof}
See Appendix \ref{appendix:proof_ball_segment_length} for the proof.
\end{proof}

Then, a sufficient condition for a $d_{\text{th}}$-loop-free path is given as follows.
\begin{lemma}[$d_{\text{th}}$-loop-free path] \label{lemma:dth_loop_free}
Consider a planar path \linebreak $r(s)=(x(s), y(s))$ parameterized by arc length, i.e., $s$ denotes the arc length.
Let $\kappa$ be the maximum curvature of the path.
The path is $d_{\text{th}}$-loop-free if it satisfies $\kappa < 1/d_{\text{th}}$ and 
\begin{align*}
\|r(s)-r(s-d)\| > d_{\text{th}}, 
\end{align*}
for $d> \frac{1}{\kappa}\sin^{-1}\left(\kappa d_{\text{th}}\right)$ and for all $s$.
\end{lemma}

\begin{proof}
From Lemma \ref{lemma:ball_segment_length}, we know that if $\kappa < 1/d_{\text{th}}$, the path cannot loop within the ball, preventing the condition of Fig. \ref{fig:d_th_loop_free_balls_and_ball_characterization}a.
Moreover, from Lemma \ref{lemma:ball_segment_length}, it can easily be confirmed that 
 $\|r(s)-r(s-d)\|$ for  $d> \frac{1}{\kappa}\sin^{-1}\left(\kappa d_{\text{th}}\right)$ is the euclidean distance from the center to a point on the part of the path that has left the ball.
Thus, if $\|r(s)-r(s-d)\| > d_{\text{th}}$, for $d> \frac{1}{\kappa}\sin^{-1}\left(\kappa d_{\text{th}}\right)$ the path cannot re-enter the ball (i.e., scenario of Fig. \ref{fig:d_th_loop_free_balls_and_ball_characterization}b is not possible).
\end{proof}

\begin{remark}
Any path can be reparameterized by arc length.
Details on this can be found in \cite{eberly2008moving}.
\end{remark}

\begin{figure*}
    \centering
    \includegraphics[width=1\linewidth]{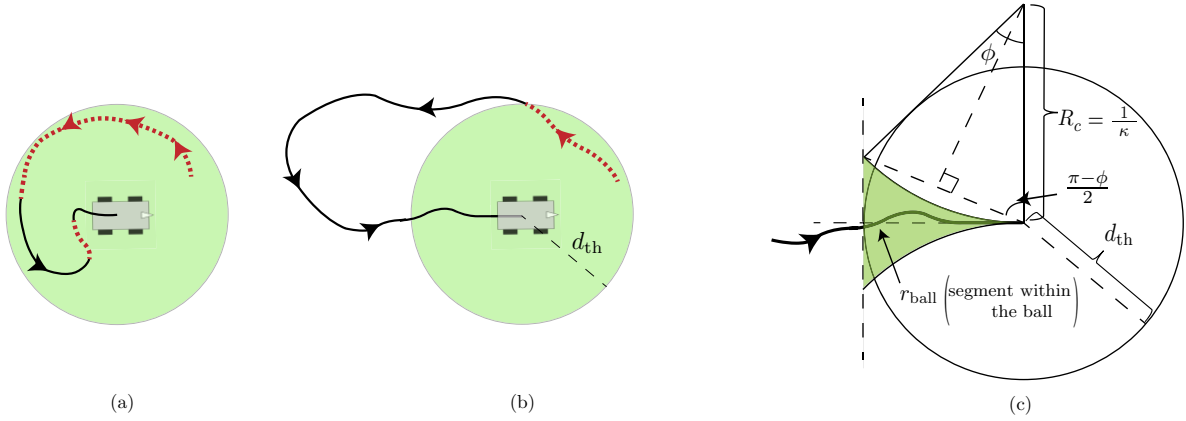}
    \caption{\small (a)-(b) Two scenarios of $d_{\text{th}}$-looping: (a) path loops within the ball and (b) path loops back to re-enter the ball. The parts causing the loop in either scenario is denoted by the dashed red line.
    (c) A path of maximum curvature $\kappa$ would lie within the shaded area.
    A sample such path is shown.}
    \label{fig:d_th_loop_free_balls_and_ball_characterization}
\end{figure*}

We next characterize the similarity or dissimilarity between the true distribution $p(\gamma_{\text{SH},-0}| \gamma_{\text{SH},-1}, \cdots, \gamma_{\text{SH},-n})$ and its Markov approximation $p(\gamma_{\text{SH},-0}| \gamma_{\text{SH},-1}) $ using the KL divergence metric.
We then utilize this to obtain sufficient conditions on the ball radius and the curvature of a path for the approximate Markovian nature to hold.

\subsubsection{Approximately-Markovian: KL Divergence metric}\label{subsubsec:approx_markov_kl_divergence} 
Consider a path as shown in Fig. \ref{fig:path_rolling_ball_and_discretized} (top).
Let $\Gamma_{\text{SH},-0}$ be the channel shadowing power on the current location and \linebreak $\Gamma_{\text{SH},-1},\cdots,\Gamma_{\text{SH},-n}$ be the channel shadowing power on the past $n$ locations along the path.
From Section \ref{subsec:channel_model}, we know that $\Gamma_{\text{SH},-0}, \cdots, \Gamma_{\text{SH},-n}$ are jointly Gaussian random variables.
The distribution of $\Gamma_{\text{SH},-0}|\Gamma_{\text{SH},-1},\cdots,\Gamma_{\text{SH},-n}$ is then given as $\mathcal{N}\left(m, \sigma^2\right)$, where 
\begin{align}
m &= \Sigma_{0,1:n}^{T} \Sigma_{1:n}^{-1} \Gamma_{\text{SH},-1:n}, \label{eq:est_mean}\\
\sigma^2 &= \sigma_{\text{SH}}^2 - \Sigma_{0,1:n}^{T} \Sigma_{1:n}^{-1}\Sigma_{0,1:n}, \label{eq:est_var}
\end{align}
with $\Gamma_{\text{SH},-1:n} = [\Gamma_{\text{SH},-1},\cdots,\Gamma_{\text{SH},-n}]^{T}$, $\Sigma_{0,1:n} = \linebreak \mathbb{E}[\Gamma_{\text{SH},-0}\Gamma_{\text{SH},-1:n}]$ and
$\Sigma_{1:n} = \mathbb{E}[\Gamma_{\text{SH},-1:n}\Gamma_{\text{SH},-1:n}^{T}]$ (see 10.5 of \cite{kay1993fundamentals}).
Moreover, $\mathbb{E}[\Gamma_{\text{SH},-i}\Gamma_{\text{SH},-j}] = \linebreak\sigma_{\text{SH}}^{2}e^{-\|q_i-q_j\|/\beta_{\text{SH}}}$, where $q_i \in \mathbb{R}^2$ is the location corresponding to $\Gamma_{\text{SH},-i}$.
Let $\alpha =  \Sigma_{1:n}^{-1} \Sigma_{0,1:n} $ denote the coefficients of the mean.
We then have $m = \alpha^{T}\Gamma_{\text{SH},-1:n} = \alpha_1\Gamma_{\text{SH},-1} + \cdots + \alpha_n\Gamma_{\text{SH},-n}$.

We want to approximate this distribution with the Markovian distribution $\Gamma_{\text{SH},-0}|\Gamma_{\text{SH},-1} \sim \mathcal{N}\left(\hat{m}, \hat{\sigma}^2\right)$ where $\hat{m} = \rho\Gamma_{\text{SH},-1}$ and $\hat{\sigma}^2 = \sigma_{\text{SH}}^2(1-\rho^2)$, with $\rho=e^{-\Delta d/\beta_{\text{SH}}}$, and $\Delta d$ being the step size of the path.
We first characterize the difference between the means, given as $\Delta m = m-\hat{m} = \Delta \alpha^{T} \Gamma_{\text{SH},-1:n}$, where $\Delta \alpha = [\alpha_1-\rho, \alpha_2, \cdots, \alpha_n]^{T}$.
$\Delta m$ is thus a zero-mean Gaussian random variable $\mathcal{N}\left(0, \sigma_{\Delta m}^{2}\right)$, where 
\begin{align}\label{eq:del_m_var}
\sigma_{\Delta m}^2 = \Delta \alpha^{T}\Sigma_{1:n}\Delta\alpha.
\end{align}

We will compare how close the true distribution and its approximation are using the KL divergence metric.
We first review the definition of KL divergence.
\begin{definition}[KL Divergence](\cite{cover2012elements})
The KL divergence between two distributions $p(x)$ and $\tilde{p}(x)$ is defined as 
\begin{align*}
KL = \int p(x)\log_{e} \frac{p(x)}{\tilde{p}(x)} dx.
\end{align*}
\end{definition}
KL divergence is a measure of the distance between two distributions (\cite{cover2012elements}).
We will utilize this as a measure of the goodness of the approximation:
the smaller the KL divergence, the better the approximation.
The following lemma gives us the expression for this KL divergence.
\begin{lemma}\label{lemma:KL_divergence}
The KL divergence between $\mathcal{N}(m, \sigma^2)$ and its approximation $\mathcal{N}(\hat{m}, \hat{\sigma}^2)$ is given as \begin{align}\label{eq:KL_divergence}
KL = \frac{\sigma_{\Delta m}^2}{2\hat{\sigma}^2}\chi_1^{2} + \frac{1}{2}\left(\frac{\sigma^2}{\hat{\sigma}^2} - 1 - \log_{e}\frac{\sigma^2}{\hat{\sigma}^2}\right),
\end{align}
where $\chi_1^{2} = (m-\hat{m})^{2}/\sigma_{\Delta m}^2$.
\end{lemma}
\begin{proof}
See \cite{robert1996intrinsic} for the proof.
\end{proof}
Since $m$ and $\hat{m}$ are functions of $\Gamma_{\text{SH},-1},\cdots,\Gamma_{\text{SH},-n}$, they are random variables.
Thus, $\chi_1^{2}$ becomes a Chi-squared random variable with one degree of freedom since $(m-\hat{m}) \sim (0, \sigma_{\Delta m}^2)$ (\cite{lancaster2005chi}),
and the KL divergence of (\ref{eq:KL_divergence}) becomes a random variable.
More specifically, from (\ref{eq:KL_divergence}), we know that the KL divergence is a scaled Chi-squared random variable with an offset term.
We use the mean $m_{KL}$ and the standard deviation $\sigma_{KL}$ of the KL divergence to capture the deviation of the Markov approximation from the true distribution.
The smaller these values are, the better the approximation is.
In our approach, we set maximum tolerable values for the mean and the standard deviation as $\epsilon_{m}$ and $\epsilon_{\sigma}$ respectively.
Then, we say that the distribution is approximately-Markovian for the parameters $\epsilon_{m}$ and $\epsilon_{\sigma}$ if we satisfy $m_{KL} \leq \epsilon_{m}$ and $\sigma_{KL} \leq \epsilon_{\sigma}$.

We next consider the setting with $3$ points in space, as shown in Fig. \ref{fig:3_points_analysis_with_curvature}a,
where we have the current point ($\Gamma_{\text{SH},-0}$), the previous point ($\Gamma_{\text{SH},-1}$) and a general point in space ($\Gamma_{\text{SH},r}$).
We are interested in mathematically characterizing the impact of $\Gamma_{\text{SH},r}$ on the estimate at the current point,
i.e., how good an approximation $\Gamma_{\text{SH},-0}|\Gamma_{\text{SH},-1} \sim \mathcal{N}(\hat{m}, \hat{\sigma}^2)$ is for the true distribution $\Gamma_{\text{SH},-0}|\Gamma_{\text{SH},-1},\Gamma_{\text{SH},r}\sim \linebreak \mathcal{N}(m, \sigma^2)$.
As we shall see, we will utilize this analysis in such a way that it serves as a good proxy for the general $n$ point analysis.
Specifically, we will utilize it to obtain bounds on the ball radius as well as on the maximum allowed curvature of a path in Section \ref{subsubsec:ball_radius_constraint} and Section \ref{subsubsec:curvature_constraint} respectively.
Let $d_1 = \|q_0-q_1\|$, $d_r = \|q_0-q_r\|$, and $d_{1r} = \|q_1-q_r\|$, as shown in Fig. \ref{fig:3_points_analysis_with_curvature}a, where $q_r$ is the location of the general point.
Moreover, $d_1 = \Delta d$.

The following lemma characterizes the mean and standard deviation of the KL divergence between the true distribution and its approximation for the $3$ point analysis.
\begin{lemma}\label{lemma:mean_std_KL_divergence}
The mean and standard deviation of the KL divergence between the true distribution $\mathcal{N}(m, \sigma^2)$ and its approximation $\mathcal{N}(\hat{m}, \hat{\sigma}^2)$ for the $3$ point analysis of Fig. \ref{fig:3_points_analysis_with_curvature}a is given as 
\begin{align*}
m_{KL} &= -\frac{1}{2}\log_{e}\left(1 - \frac{\sigma_{\Delta m}^2}{\hat{\sigma}^2}\right),\\
\sigma_{KL} &= \frac{\sigma_{\Delta m}^2}{\sqrt{2}\hat{\sigma}^2},
\end{align*}
where
\begin{align*}
\sigma_{\Delta m}^2 = \sigma_{\text{SH}}^2\frac{\left(e^{-d_r/\beta_{\text{SH}}} - e^{-(d_1+d_{1r})/\beta_{\text{SH}}}\right)^2}{1 - e^{-2d_{1r}/\beta_{\text{SH}}}}.
\end{align*} 
\end{lemma}

\begin{proof}
See Appendix \ref{appendix:proof_KL_divergence_3_points} for the proof.
\end{proof}

\begin{figure}
    \centering
    \includegraphics[width=0.95\linewidth]{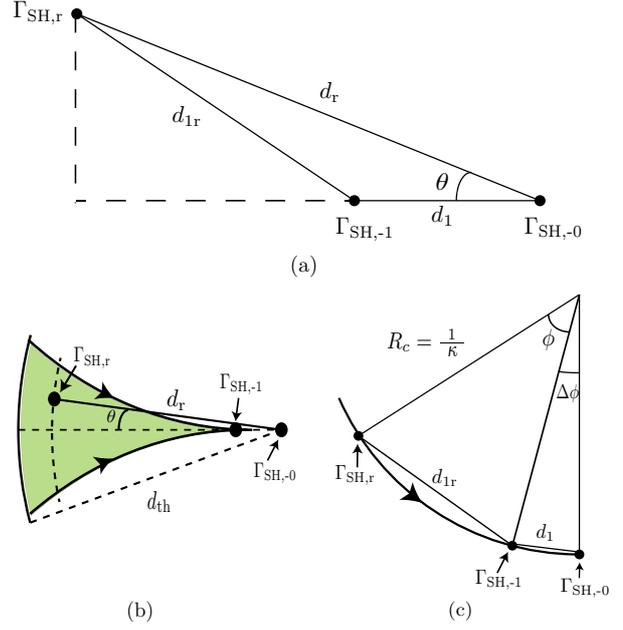}
    \caption{\small $3$ points analysis (a) for a general case, (b) for a path with maximum curvature $\kappa$ that satisfies $\kappa<1/d_{\text{th}}$, and (c) along a path with a constant curvature.}
    \label{fig:3_points_analysis_with_curvature}
\end{figure}

Any point on the path can belong to three possible regions:
$1$) the shaded region within the ball of Fig. \ref{fig:d_th_loop_free_balls_and_ball_characterization}c,
$2$) within the  ball but outside the shaded region, and
$3$) outside the ball.
If the path is $d_{\text{th}}$-loop-free, then no point of the path lies within region $2$ (i.e., within the ball but outside the shaded region).
We next characterize the minimum ball radius and the maximum allowed curvature of a path such that the impact of any point ($\Gamma_{\text{SH},r}$) in region $1$ and $3$ on the estimate at the center of the ball is negligible.

\subsubsection{Ball Radius}\label{subsubsec:ball_radius_constraint}
We next utilize our analysis to determine the ball radius $d_{\text{th}}$.
We wish to select the minimum $d_{\text{th}}$ such that the impact of any point outside the ball on the approximation is within the tolerable KL divergence parameters $\epsilon_{m}$ and $\epsilon_{\sigma}$, i.e., the KL divergence between the true and the approximating distribution (in the $3$ point analysis) satisfies $m_{KL} \leq \epsilon_{m}$ and $\sigma_{KL} \leq \epsilon_{\sigma}$.

The following lemma characterizes what the minimum ball radius $d_{\text{th}}$ should be.
\begin{lemma}\label{lemma:ball_radius}
The minimum ball radius $d_{\text{th}}$ such that any point outside the ball satisfies the maximum tolerable KL divergence parameters $\epsilon_{m}$ and $\epsilon_{\sigma}$ for the $3$ point analysis, is given by
\begin{align*}
d_{\text{th}} & = \frac{\beta_{\text{SH}}}{2} \log_e\left(\rho^2 + \frac{1 - \rho^{2}}{\epsilon_d}\right),
\end{align*}
where $\rho = e^{-\Delta d/\beta_{\text{SH}}}$ and $\epsilon_d = \min\left\{1-e^{-2\epsilon_{m}}, \sqrt{2}\epsilon_{\sigma}\right\}$.
\end{lemma}

\begin{proof}
See Appendix \ref{appendix:proof_ball_radius} for the proof.
\end{proof}

\subsubsection{Curvature Constraint}\label{subsubsec:curvature_constraint}
We next utilize the $3$ point analysis to determine the maximum curvature of a path such that it is approximately-\linebreak Markovian, i.e., it satisfies the KL divergence constraint \linebreak $m_{KL} \leq \epsilon_{m}$ and $\sigma_{KL} \leq \epsilon_{\sigma}$.

Consider the scenario in Fig. \ref{fig:3_points_analysis_with_curvature}b.
For a given maximum curvature $\kappa$, any valid point of the path must lie within the shaded region of the figure, where the boundary corresponds to circular paths with curvature $\kappa$.
We wish to find the maximum allowed curvature such that the impact of any point within the shaded region on the approximation is within the tolerable KL divergence parameters $\epsilon_{m}$ and $\epsilon_{\sigma}$, i.e., the KL divergence between the true and the approximating distribution (in the $3$ point analysis) satisfies $m_{KL} \leq \epsilon_{m}$ and $\sigma_{KL} \leq \epsilon_{\sigma}$.
The following lemma characterizes this maximum allowed curvature as the solution of an optimization problem.

\begin{lemma}\label{lemma:curv_constraint}
The maximum allowed curvature $\kappa_{\text{th}}$ such that any past point on the path within the ball of radius $d_{\text{th}}$ satisfies the maximum tolerable KL divergence parameters $\epsilon_{m}$ and $\epsilon_{\sigma}$ for the $3$ point analysis, is the solution to the following optimization problem:

\begin{equation}\label{eq:curvature_optimization}
\begin{aligned}
& \text{maximize} & & \kappa\\
& \text{subject to } & & \max_{\phi:0< \phi \leq h_{\text{cons}}(\kappa)} 
 h_{\text{opt}}(\kappa, \phi) \leq \epsilon_d \\
& & & \kappa < 1/d_{\text{th}},
\end{aligned}
\end{equation}
where 
\begin{align*}
h_{\text{opt}}(\kappa, \phi) = \frac{\left( e^{-\frac{2}{\kappa\beta_{\text{SH}}}\sin(\frac{\phi + \Delta\phi}{2})} -  \rho e^{-\frac{2}{\kappa\beta_{\text{SH}}}\sin(\frac{\phi}{2})} \right)^2}
{(1 - e^{-\frac{4}{\kappa\beta_{\text{SH}}}\sin(\frac{\phi}{2})})(1-\rho^2)},
\end{align*}
and $h_{\text{cons}}(\kappa) = 2\sin^{-1}(\frac{\kappa d_{\text{th}}}{2})-\Delta \phi$, $\Delta \phi = 2\sin^{-1}(\frac{\kappa \Delta d}{2})$, $\rho=e^{-\Delta d/\beta_{\text{SH}}}$, $\epsilon_d = \min\left\{1-e^{-2\epsilon_{m}}, \sqrt{2}\epsilon_{\sigma}\right\}$.

\end{lemma}

\begin{proof}
See Appendix \ref{appendix:proof_curv_constraint} for the proof.
\end{proof}

\begin{remark}
Ideally, we would have preferred to use the KL divergence between the approximation and the true distribution where we condition on all the past points on the path within the ball radius, as opposed to using just the point with the maximal impact.
However, such an analysis does not lend itself to a neat characterization of the maximum allowed curvature.
Through simulations, we have seen that the $3$ point analysis, as described in Lemma \ref{lemma:curv_constraint}, serves as a good proxy for the $n$ past points case on a circular path (which has a maximum curvature everywhere for a given $\kappa$).
For instance, for parameters $\kappa = 1/15$, $\Delta d = 0.1$ and $\beta_{\text{SH}}=5$ m, the KL divergence mean and standard deviation when considering all the past points of the path within the ball are $m_{KL}=6\times 10^{-7}$ and $\sigma_{KL} = 9\times 10^{-7}$ respectively.
This is comparable to the values $m_{KL}=3\times 10^{-7}$ and $\sigma_{KL} = 5\times 10^{-7}$ obtained for the $3$ point analysis from Lemma \ref{lemma:curv_constraint}.
\end{remark}

Finally, we put together all our results to provide sufficient conditions for an approximately-Markovian path.
\begin{lemma}[Approximately-Markovian Path]\label{lemma:approx_markovian}
Let  $r(s) = (x(s), y(s))$ be a path parameterized by its arc length.
The path is approximately-Markovian for given maximum tolerable KL divergence parameters $\epsilon_m$ and  $\epsilon_{\sigma}$ for the $3$ point analysis, if it satisfies the following conditions:
\begin{enumerate}
\item $r(s)$ is $d_{\text{th}}$-loop-free for ball radius $d_{\text{th}}$ (as characterized by Lemma \ref{lemma:dth_loop_free}),
\item curvature $\kappa(s) < \kappa_{\text{th}}$ for all $s$,
\end{enumerate}
where $d_{\text{th}}$ and $\kappa_{\text{th}}$ are obtained from Lemma \ref{lemma:ball_radius} and Lemma \ref{lemma:curv_constraint} respectively.
\end{lemma}

Consider a given path.
For a given $\epsilon_m$ and $\epsilon_{\sigma}$, we can check if the path satisfies the conditions of Lemma \ref{lemma:approx_markovian}.
If it does, we can then directly use the results of Section \ref{subsec:straight_path_without_mp} to obtain the PDF of the FPD for the path.
Note that even if the path does not satisfy the conditions, the path may still be approximately-Markovian as the conditions of Lemma \ref{lemma:approx_markovian} are sufficient conditions.

\section{Characterizing FPD Considering Multipath}\label{sec:with_multipath}

The previous section analyzed the FPD to the connectivity threshold when the multipath component was ignored.
In this section, we show how to derive the FPD density in the presence of the multipath fading component, and for the most general channel model of $\Gamma(d) = \gamma_{\text{PL}}(d) + \Gamma_{\text{SH}}(d) + \Gamma_{\text{MP}}(d)$.
We begin by analyzing straight paths in Section \ref{subsec:straight_path_with_mp}, where we derive the PDF of the FPD using a recursive formulation.
We then extend our analysis to a larger space of paths in Section \ref{subsec:approx_markov_paths_with_mp}.

\subsection{Straight Paths: A Recursive Characterization}\label{subsec:straight_path_with_mp}
We first characterize the PDF of the distance traveled until connectivity for straight paths.
We consider the scenario described in Section \ref{subsec:straight_path_without_mp},
where a robot situated at a distance $d_{\text{src}}$ from a remote operator to which it needs to be connected,
moves in a straight path in the direction specified by the angle $\theta_{\text{src}}$, as shown in Fig. \ref{fig:setup_general_straight}b.
$\Gamma(d)$ represents the channel power when the robot is at distance $d$ along direction $\theta_{\text{src}}$, as marked in Fig. \ref{fig:setup_general_straight}b.

Recall that we define connectivity as the event where $\Gamma(d) \geq \gamma_{\text{th}}$.
The connectivity requirement is then given as $\Gamma(d) = \gamma_{\text{PL}}(d) + \Gamma_{\text{SH}}(d) + \Gamma_{\text{MP}}(d) \geq \gamma_{\text{th}}$, considering all the channel components.
In this case, the approach of Section \ref{subsec:straight_path_without_mp} is not applicable anymore as we no longer deal with a Markov process.
Even if the multipath component was taken to be a Gauss-Markov process (which could be a valid model for some environments (\cite{hashemi1994study})), the resultant channel power would not be Markovian, as can be verified from Lemma \ref{lemma:gauss_markov_cond}.
In this section, we assume that the robot measures the channel along the chosen straight path in discrete steps of size $\Delta d$.
We assume that $\Delta d$ is such that the multipath random variable is uncorrelated at the distance $\Delta d$ apart (this is a realistic assumption as multipath decorrelates fast (\cite{malmirchegini2012spatial})).
We then index the channel power and shadowing components accordingly, i.e., let $\Gamma_k = \Gamma(k\Delta d)$ and $\Gamma_{\text{SH},k} = \Gamma_{\text{SH}}(k\Delta d)$.
The probability of failure of connectivity at the end of $N$ steps (given the initial failure of connectivity) can then be written as
\begin{multline} \label{eq:desired_prob_sdp}
\mathrm{Pr}\left(\Gamma_1, \Gamma_2, \cdots, \Gamma_{N} < \gamma_{\text{th}}| \Gamma_0 < \gamma_{\text{th}} \right) \\
= \underset{\gamma_1,\cdots, \gamma_N < \gamma_{\text{th}}}{\int\cdots\int}p(\gamma_1,\cdots,\gamma_N|\Gamma_0 < \gamma_{\text{th}}) \mathrm{d}\gamma_1 \cdots \mathrm{d} \gamma_{N},
\end{multline}
where $p(\gamma_1,\cdots,\gamma_N|\Gamma_0<\gamma_{\text{th}})$ is the conditional joint density function of $\Gamma_1,  \cdots, \Gamma_{N}$.
Consider the computation of this integral, which is an integration in an $N$ dimensional space.
If we discretize the domain of $\Gamma_k$ into $M$ parts, then a direct computation of the FPD for upto $N$ steps would have a computational complexity of $O(NM^N)$, which is infeasible for high values of $M$ and $N$.
Instead, we show how this can be solved efficiently through a recursive integral computation in $O(NM\log(M))$.
In contrast, our previously proposed dynamic programming approach of \linebreak \cite{muralidharan2017first} had a computational complexity of $O(N^2M^2)$.

As mentioned before, the robot measures the channel in discrete steps of size $\Delta d$.
Let $d_k = k\Delta d$ denote the distance when $k$ steps are taken.
Then, it can be shown, using (\ref{eq:mean_var_trans_pdf}), that the shadowing component is an autoregressive AR($1$) process, the continuous analogue of which is the Ornstein-Uhlenbeck process (note that the shadowing component is Markovian):
\begin{align*}
\Gamma_{\text{SH},k+1} = \rho\Gamma_{\text{SH},k} + \sigma_{\text{SH}}\sqrt{1- \rho^2}Z_{k},
\end{align*}
where $\rho = e^{-\Delta d/\beta_{\text{SH}}}$ and $Z_k$ are i.i.d. with a standard normal distribution.
The conditional random variable $\Gamma_{\text{SH},k+1}|\gamma_{\text{SH},k}$ is thus a Gaussian random variable with mean $\rho \gamma_{\text{SH},k}$ and variance $\sigma_{\text{SH}}^2 (1-\rho^2)$.

Note that the desired probability of (\ref{eq:desired_prob_sdp}) can be expressed as
\begin{align}\label{eq:cond_desired_prob_alt}
\mathrm{Pr}\left(\Gamma_1,\cdots, \Gamma_N<\gamma_{\text{th}}|\Gamma_0 < \gamma_{\text{th}}\right)  = \frac{\mathrm{Pr}\left(\Gamma_0, \cdots, \Gamma_N < \gamma_{\text{th}} \right)}{\mathrm{Pr}\left(\Gamma_0 < \gamma_{\text{th}}\right)}.
\end{align}

We next show how to compute $\mathrm{Pr}(\Gamma_0, \Gamma_1, \cdots \Gamma_N < \gamma_{\text{th}})$ via a recursive characterization.
This is inspired in part by the calculation of orthant probabilities for auto-regressive sequences in \cite{craig2008new}.
Define the set of functions $\mathcal{J}_{k}$, as follows:

\vspace{0.15in}
\begin{align}\label{eq:cost_to_go}
& \mathcal{J}_{k}(\gamma_{\text{SH},k})  = \int_{\gamma_{\text{MP},k}=-\infty}^{\gamma_{\text{th}} - \gamma_{\text{PL}}(d_k) - \gamma_{\text{SH},k}}\nonumber\\
& \;\;\;\;\;\;\;\;\;\;\;\;\;\; \times \underset{S_{k-1}}{\int \cdots \int}p(\gamma_{\text{SH},0},\gamma_{\text{MP},0},\cdots,\gamma_{\text{SH},k}, \gamma_{\text{MP},k}) \nonumber \\
& \;\;\;\;\;\;\;\;\;\;\;\;\;\times \mathrm{d}\gamma_{\text{SH},0}\mathrm{d}\gamma_{\text{MP},0}\cdots \mathrm{d}\gamma_{\text{SH},k-1}\mathrm{d}\gamma_{\text{MP},k-1} \mathrm{d}\gamma_{\text{MP},k}
\end{align}
where $S_{k-1} = \cap_{i=0}^{k-1}\{\gamma_{\text{SH},i}, \gamma_{\text{MP},i}:\gamma_{\text{PL}}(d_i) + \gamma_{\text{SH},i} + \gamma_{\text{MP},i} < \gamma_{\text{th}}\}$ and $p(\gamma_{\text{SH},0},\gamma_{\text{MP},0},\cdots,\gamma_{\text{SH},k}, \gamma_{\text{MP},k})$ is the joint density of $\Gamma_{\text{SH},0},\Gamma_{\text{MP},0}, \cdots ,\Gamma_{\text{SH},k},\Gamma_{\text{MP},k}$.
Note that 
\begin{align}\label{eq:cond_desired_prob_sdp}
&\mathrm{Pr}\left(\Gamma_0, \Gamma_1, \cdots, \Gamma_N < \gamma_{\text{th}} \right) \nonumber\\
& \;\;\;\;\;\;\;\;\;\;\;\;= \underset{S^{N}}{\int \cdots \int}p(\gamma_{\text{SH},0},\gamma_{\text{MP},0},\cdots,\gamma_{\text{SH},N}, \gamma_{\text{MP},N}) \nonumber \\
& \;\;\;\;\;\;\;\;\;\;\;\;\;\;\;\;\;\;\;\;\;\;\;\;\;\;\;\; \times \mathrm{d}\gamma_{\text{SH},0}\mathrm{d}\gamma_{\text{MP},0}\cdots \mathrm{d}\gamma_{\text{SH},N}\mathrm{d}\gamma_{\text{MP},N} \nonumber \\
& \;\;\;\;\;\;\;\;\;\;\;\;= \int_{\gamma_{\text{SH},N}=-\infty}^{\infty}\mathcal{J}_{N}(\gamma_{\text{SH},N})\mathrm{d}\gamma_{\text{SH},N}.
\end{align}
In the following lemma we show how to compute $\mathcal{J}_{k}(\gamma_{\text{SH},k})$ recursively.

\begin{lemma}\label{lemma:recursive_characterization}
The functions $\mathcal{J}_{k}$, for $k =1, \cdots, N$, of (\ref{eq:cost_to_go}) can be computed by the recursion:
\begin{align*}
\mathcal{J}_{k+1}(\gamma_{\text{SH},k+1}) &= F_{\text{MP}}(\gamma_{\text{th}} - \gamma_{\text{PL}}(d_{k+1})-\gamma_{\text{SH},k+1})\\
& \;\;\;\times \frac{1}{\rho}\int_{u=-\infty}^{\infty} \varphi\left(\frac{\gamma_{\text{SH},k+1} - u}{\sigma_{\text{SH}}\sqrt{1-\rho}}\right)\mathcal{J}_{k}(\frac{u}{\rho})\mathrm{d}u,
\end{align*}
initialized with 
\begin{align*}
\mathcal{J}_{0}(\gamma_{\text{SH},0}) = F_{\text{MP}}(\gamma_{\text{th}} - \gamma_{\text{PL}}(0)-\gamma_{\text{SH},0})\varphi\left(\frac{\gamma_{\text{SH},0}}{\sigma_{\text{SH}}}\right),
\end{align*}
where $F_{\text{MP}}(.)$ is the CDF of the multipath random variable $\Gamma_{\text{MP}}$ and $\varphi(.)$ is the standard Gaussian density function.
\end{lemma}

\begin{proof}
It can be seen that this clearly holds for $k=0$:

\begin{align*}
\mathcal{J}_{0}(\gamma_{\text{SH},0}) & = \int_{\gamma_{\text{MP},k}=-\infty}^{\gamma_{\text{th}} - \gamma_{\text{PL}}(d_0) - \gamma_{\text{SH},0}}p(\gamma_{\text{SH},0},\gamma_{\text{MP},0}) 
\mathrm{d}\gamma_{\text{MP},0}\\
& = F_{\text{MP}}(\gamma_{\text{th}} - \gamma_{\text{PL}}(0)-\gamma_{\text{SH},0})\varphi\left(\frac{\gamma_{\text{SH},0}}{\sigma_{\text{SH}}}\right).
\end{align*}
Next, $\mathcal{J}_{k+1}(\gamma_{\text{SH},k+1})$ can be expanded as
\begin{align*}
& \mathcal{J}_{k+1}(\gamma_{\text{SH},k+1}) \\
& = 
\int_{-\infty}^{\gamma_{\text{th,MP},k+1}} \underset{S_{k}}{\int \cdots \int}p(\gamma_{\text{SH},0},\gamma_{\text{MP},0},\cdots,\gamma_{\text{SH},k+1}, \gamma_{\text{MP},k+1}) \nonumber \\
& \;\;\;\;\;\;\;\;\;\;\;\;\;\;\;\;\;\;\;\;\;\times \mathrm{d}\gamma_{\text{SH},0}\mathrm{d}\gamma_{\text{MP},0}\cdots \mathrm{d}\gamma_{\text{SH},k}\mathrm{d}\gamma_{\text{MP},k} \mathrm{d}\gamma_{\text{MP},k+1}\\
&  = \int_{-\infty}^{\gamma_{\text{th,MP},k+1}}p(\gamma_{\text{MP},k+1})\mathrm{d}\gamma_{\text{MP},k+1}\int_{-\infty}^{\infty}p(\gamma_{\text{SH},k+1}|\gamma_{\text{SH},k})\\
& \;\;\;\;\;\;\;\; \times \int_{-\infty}^{\gamma_{\text{th,MP},k}} \underset{S_{k-1}}{\int \cdots \int}p(\gamma_{\text{SH},0},\gamma_{\text{MP},0},\cdots,\gamma_{\text{SH},k}, \gamma_{\text{MP},k}) \nonumber \\
& \;\;\;\;\;\;\;\;\;\;\;\;\;\;\;\;\;\;\times \mathrm{d}\gamma_{\text{SH},0}\mathrm{d}\gamma_{\text{MP},0}\cdots \mathrm{d}\gamma_{\text{SH},k-1}\mathrm{d}\gamma_{\text{MP},k-1} \mathrm{d}\gamma_{\text{MP},k}\\
&  = F_{\text{MP}}(\gamma_{\text{th,MP},k+1}) \\
&  \;\;\;\;\;\;\;\;\;\;\;\;\; \times \int_{-\infty}^{\infty} \varphi\left(\frac{\gamma_{\text{SH},k+1} - \rho \gamma_{\text{SH},k}}{\sigma_{\text{SH}}\sqrt{1-\rho}}\right)\mathcal{J}_{k}(\gamma_{\text{SH},k})\mathrm{d}\gamma_{\text{SH},k}\\
&  =  \frac{F_{\text{MP}}(\gamma_{\text{th,MP},k+1})}{\rho}\int_{u=-\infty}^{\infty} \varphi\left(\frac{\gamma_{\text{SH},k+1} - u}{\sigma_{\text{SH}}\sqrt{1-\rho}}\right)\mathcal{J}_{k}(\frac{u}{\rho})\mathrm{d}u,
\end{align*}
where $\gamma_{\text{th,MP},k} = \gamma_{\text{th}} - \gamma_{\text{PL}}(d) - \gamma_{\text{SH},k}$.
\end{proof}
\begin{remark}
Note that the recursive integral in Lemma \ref{lemma:recursive_characterization} is in the form of a convolution.
This can be computed efficiently using the Fast Fourier transform.
\end{remark}
Using Lemma \ref{lemma:recursive_characterization}, we can compute $\mathrm{Pr}\left(\Gamma_0, \Gamma_1, \cdots, \Gamma_N < \gamma_{\text{th}} \right)$ as shown in (\ref{eq:cond_desired_prob_sdp}), which in turn is used to compute \linebreak $\mathrm{Pr}\left(\Gamma_1, \cdots \Gamma_{N} < \gamma_{\text{th}}|\Gamma_{0}<\gamma_{\text{th}}\right)$ via (\ref{eq:cond_desired_prob_alt}).

Next, we use this result to calculate the FPD probability.
Let $\mathcal{K} = \min_{k=1,2,\cdots}\left\{k:\Gamma_k \geq \gamma_{\text{th}}, \Gamma_0 < \gamma_{\text{th}}\right\}$ be the random variable which denotes the upcrossing first passage step to connectivity given that $\Gamma_0$ is restricted to lie below $\gamma_{\text{th}}$.
Then,
\begin{align*}
\mathrm{Pr}(\mathcal{K}=k) & = \mathrm{Pr}\left(\Gamma_1, \cdots \Gamma_{k-1} < \gamma_{\text{th}}, \Gamma_{k}\geq \gamma_{\text{th}}|\Gamma_{0}<\gamma_{\text{th}}\right)\\
& = \mathrm{Pr}\left(\Gamma_1, \cdots \Gamma_{k-1} < \gamma_{\text{th}}|\Gamma_{0}<\gamma_{\text{th}}\right)\\
&\;\;\;\;\;\;\;\;\;\;\;\;\;\;\;\;\;\;\;- \mathrm{Pr}\left(\Gamma_1, \cdots \Gamma_{k} < \gamma_{\text{th}}|\Gamma_{0}<\gamma_{\text{th}}\right),
\end{align*}
where both terms on the right hand side can be obtained from our recursive characterization using Lemma \ref{lemma:recursive_characterization}.

\subsection{Approximately-Markovian Paths}\label{subsec:approx_markov_paths_with_mp}

In this section, we characterize the space of paths (beyond straight paths) for which we can characterize the statistics of the distance traveled until connectivity.
As we saw in Section \ref{subsec:straight_path_with_mp}, the recursive characterization of Lemma \ref{lemma:recursive_characterization} depends on the channel shadowing power being a Markov process.
Specifically, the proof of Lemma \ref{lemma:recursive_characterization} requires that 
\linebreak $p(\gamma_{\text{SH},-0}| \gamma_{\text{SH},-1}, \gamma_{\text{SH},-2}, \cdots) = p(\gamma_{\text{SH},-0}| \gamma_{\text{SH},-1}) $, where $\Gamma_{\text{SH},-0}$ is the shadowing power at the current location and \linebreak $\Gamma_{\text{SH}, -1}, \Gamma_{\text{SH},-2}, \cdots $ are the channel shadowing power at previously visited points, as shown in Fig. \ref{fig:path_rolling_ball_and_discretized} (top).
We can then directly use the tools and strategies developed in Section \ref{subsec:approx_markov_paths_without_mp} to characterize the space of paths that are approximately-Markovian.
We then obtain the statistics of the FPD for these paths using Lemma \ref{lemma:recursive_characterization}.

\begin{remark}[Computational complexity]
A natural question that arises is: \emph{why not use the results of Section \ref{subsec:straight_path_with_mp} to tackle the case without considering multipath of Section \ref{subsec:straight_path_without_mp}?}
We next address this.
As discussed in Section \ref{subsec:straight_path_with_mp}, the computation cost of Lemma \ref{lemma:recursive_characterization} for upto $N$ steps is $O(NM\log(M))$.
In contrast, the computational cost of Theorem \ref{theorem:upcrossing_fpd} for the case without considering multipath, for upto $N$ steps, is $O(N^2)$.
Since $M>>N$, the stochastic differential equation approach is more computationally efficient.
Moreover, the characterization of the $\epsilon$-upcrossing FPD of Section \ref{subsec:straight_path_without_mp} can be used for analytical purposes. 
\end{remark}

\section{Numerical Results based on Real Channel Data}\label{sec:numerical_results}

In this section, we validate the derivations of Sections \ref{sec:without_multipath} and \ref{sec:with_multipath} in a simulation environment with real channel parameters.
We also highlight interesting trends of the FPD statistics as a function of the channel parameters.
The channel is generated using the channel model described in Section \ref{subsec:channel_model}, with parameters obtained from real channel measurements in downtown San Francisco (\cite{smith2004urban}): $n_{\text{PL}} = 4.2$, $\sigma_{\text{SH}}^2 = 8.41$ and $\beta_{\text{SH}} = 12.92$ m.
We impose a minimum required received SNR of $20$ dB,
the noise power is taken to be a realistic $-100$ dBmW,
and the transmit power is taken to be $30$ dBmW, which results in a channel power connectivity threshold of $\gamma_{\text{th}} = -110$ dB.
We furthermore take the upcrossing FDP constant to be $\epsilon = 0.1$ in the simulation results.

We consider a discretization step size of $\Delta d = 0.03$ m.
Let the maximum tolerable KL divergence parameters be $\epsilon_m = 0.001$ and $\epsilon_{\sigma} = 0.001$.
Then, the ball radius $d_{\text{th}} = 9.5$ m and the maximum allowed curvature $\kappa_{\text{th}} = 1.04$  satisfy Lemma \ref{lemma:ball_radius} and Lemma \ref{lemma:curv_constraint} respectively.
We will demonstrate the efficacy of our proposed approaches through two different paths that satisfy these constraints and are thus \linebreak approximately-Markovian:
1) an archimedian spiral with equation $r_d = 11 + 5e^{\theta}$, and
2) a logarithmic spiral with equation $r_d = 11e^{0.5\theta}$,
where both equations are in polar coordinates $(r_d, \theta)$.
Figures \ref{fig:fpd_simulations}a and \ref{fig:fpd_simulations}b show the path and the curvature along the path of the archimedian spiral respectively, while
figures \ref{fig:fpd_simulations}g and \ref{fig:fpd_simulations}h show the path and the curvature along the path of the logarithmic spiral respectively.
The remote station is located at the origin as denoted in figures \ref{fig:fpd_simulations}a and \ref{fig:fpd_simulations}g.

\subsection{Results Without Considering Multipath}\label{subsec:simulations_sde}
We first consider the case without multipath.
Figures \ref{fig:fpd_simulations}c and \ref{fig:fpd_simulations}d show the PDF and CDF respectively of the upcrossing FPD for the archimedian path.
Figures \ref{fig:fpd_simulations}i and \ref{fig:fpd_simulations}j show the PDF and CDF respectively of the upcrossing FPD for the logarithmic path.
We can see that, for both paths, our theoretical derivations match the true statistics obtained via Monte Carlo simulations very well.

\begin{figure*}
    \centering
    \includegraphics[width=\linewidth]{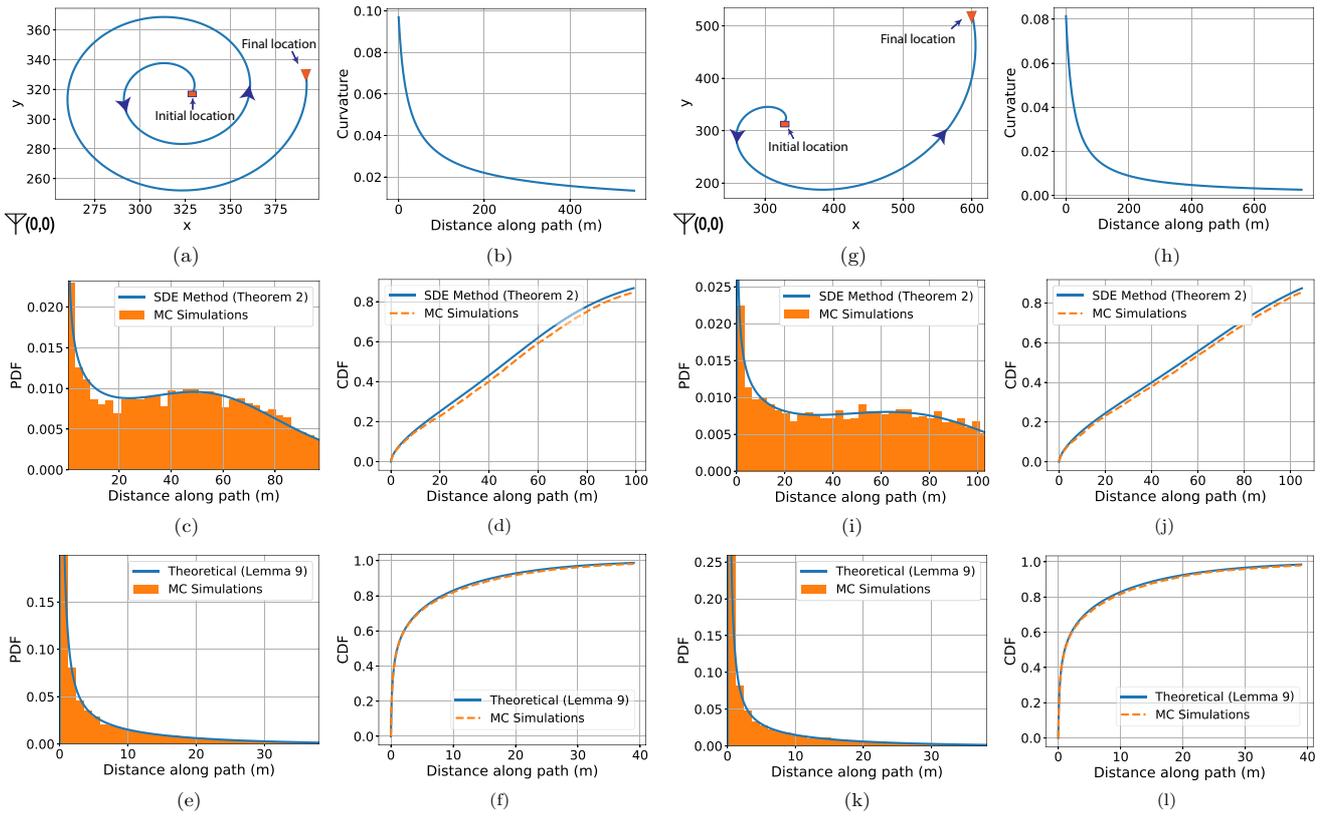}
    \caption{\small (a)-(f): (a) Archimedian spiral as the path of the robot, (b) curvature along the archimedian spiral, (c) PDF and (d) CDF of upcrossing FPD without considering multipath, (e) PDF and (f) CDF of upcrossing FPD when including multipath.\\
    (g)-(l): (g) Logarithmic spiral as the path of the robot, (h) curvature along the logarithmic spiral, (i) PDF and (j) CDF of upcrossing FPD without considering multipath, (k) PDF and (l) CDF of upcrossing FPD when including multipath.}
    \label{fig:fpd_simulations}
\end{figure*}

\subsection{Results When Including Multipath}\label{subsec:simulations_sdp}

Next, consider the case where multipath of the environment can not be neglected.
We then simulate the multipath fading as an uncorrelated Rician random variable.
Rician distribution is a common distribution for characterizing multipath (\cite{rappaport1996wireless}) and is given by
\begin{align*}
f_{\text{ric}}(z) = (1+K_{\text{ric}}) e^{-K_{\text{ric}}-(1+K_{\text{ric}})z} I_{0}\left(2\sqrt{zK_{\text{ric}}(1+K_{\text{ric}})}\right),
\end{align*}
where $I_{0}(.)$ is the modified $0^{\text{th}}$ order Bessel function and the parameter $K_{\text{ric}}$ is the ratio of the power in the line of sight component to the power in the non-line of sight components of the channel.
We use the rician parameter $K_{\text{ric}} = 1.59$, which we obtain from the real channel measurements in downtown San Francisco.
We further assume that the multipath component gets uncorrelated at our discretization interval of $0.03$ m, which is a reasonable assumption in many cases (\cite{malmirchegini2012spatial}).

Figures \ref{fig:fpd_simulations}e and \ref{fig:fpd_simulations}f show the PDF and CDF respectively of the upcrossing FPD for the archimedian path.
Figures \ref{fig:fpd_simulations}k and \ref{fig:fpd_simulations}l show the PDF and CDF respectively of the upcrossing FPD for the logarithmic path.
The histogram obtained via Monte Carlo simulations is also plotted for comparison.
It can be seen that in the case of both paths our derivations match the true statistics very well.
 
\begin{figure}
    \centering
    \includegraphics[width=1\linewidth]{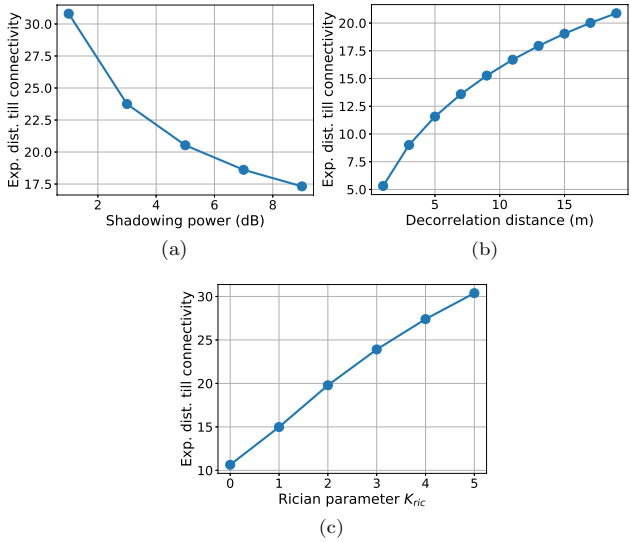}
    \caption{\small Expected distance until connectivity (with multipath) as a function of the (a) shadowing power, (b) shadowing decorrelation distance and (c) rician parameter $K_{\text{ric}}$, for the case of a straight path with  $d_{\text{src}}=550$ m and $\theta_{\text{src}} = 0$ rad.}
    \label{fig:vary_ch_params}
\end{figure}

Finally, different environments will have different underlying channel parameters. Thus, we next consider the impact of the underlying channel parameters on the FPD.
Figures \ref{fig:vary_ch_params}a and \ref{fig:vary_ch_params}b show the expected distance traveled as a function of the shadowing decorrelation distance ($\beta_{\text{SH}}$) and the shadowing variance ($\sigma^{2}_{\text{SH}}$) respectively when $d_{\text{src}} = 550$ m and $\theta_{\text{src}} = 0$ rad, along a straight path.
Increasing the shadowing power directly increases the spatial variance of the channel power. 
Thus, with a higher probability, $\Gamma(d)$ stumbles upon the connectivity threshold earlier, resulting in a smaller FPD, as can be seen.
An increase in the decorrelation distance, on the other hand, implies a greater spatial correlation of the channel power and decreases the spatial variation.
Thus, we observe that the expected traveled distance increases when increasing the decorrelation distance.
Figure \ref{fig:vary_ch_params}c shows the expected distance until connectivity as a function of $K_{\text{ric}}$ of multipath.
For large values of $K_{\text{ric}}$, the line of sight component dominates and results in a more deterministic multipath term.
Decreasing $K_{\text{ric}}$, on the other hand, results in an increase in the variance of the multipath component, thus increasing the randomness of the channel.
Thus as $K_{\text{ric}}$ decreases, $\Gamma(d)$ would cross the connectivity threshold earlier with a higher probability (due to the increase in channel randomness), resulting in a smaller expected distance traveled.

\section{Conclusions}
In this paper, we considered the scenario of a robot that seeks to get connected to another robot or a remote operator, as it moves along a path.
We started by mathematically characterizing the PDF of the distance traveled until connectivity along straight paths,
using a stochastic differential equation analysis when multipath can be ignored,
and a recursive characterization for the case of multipath.
We then developed a theoretical characterization of a more general space of paths, based on properties of the path such as its curvature, for which we can theoretically characterize the PDF of the FPD.
Our characterizations not only enable new theoretical analysis but also allow for an efficient low-complexity implementation.
Finally, we confirmed our theoretical results with simulations with real channel parameters from downtown San Francisco, and highlighted interesting trends of the FPD.

\appendix
\section{Appendix}

\subsection{Proof of Lemma \ref{lemma:ball_segment_length}}\label{appendix:proof_ball_segment_length}
\begin{proof}
Let $r(s)=(x(s), y(s))$ be the equation of the path parameterized by arc length.
Since the path is parameterized by arc length, we have 
\begin{align}\label{eq:arc_length}
\|r'(s)\|^2 = |x'(s)|^2 + |y'(s)|^2 = 1.
\end{align}
Moreover, we have the curvature constraint
\begin{align}\label{eq:curv_constraint}
\|r''(s)\|^2 = |x''(s)|^2 + |y''(s)|^2 \leq \kappa^2.
\end{align}
Let $s_0$ denote the current point, i.e., the center of the ball.
Without loss of generality, let $(x(s_0), y(s_0)) = (0,0)$ and let the tangent at $s_0$  be parallel to the x-axis, i.e., $x'(s_0)=-1$, $y'(s_0)=0$, as shown in Fig. \ref{fig:d_th_loop_free_balls_and_ball_characterization}c.

We first prove that no point of $r_{\text{ball}}$ can lie outside the shaded region of Fig. \ref{fig:d_th_loop_free_balls_and_ball_characterization}c.
Note that the shaded region has a boundary on the left corresponding to $x=-d_{\text{th}}$, and
the two other boundaries correspond to circular arcs with curvature $\kappa$.
Let us consider traveling backward along the path. 
For a given distance $d_x$ traveled along the negative x-axis (i.e., $x(s)=-d_x$), the path which maximizes the distance traveled along the y-axis $|y(s)|$, is the one that minimizes the x-axis velocity $|x'(s)|$ and maximizes the y-axis velocity $|y'(s)|$ the most.
This corresponds to the circular path $(R_c\cos(s/R_c), R_c\sin(s/R_c))$ with constant curvature $\kappa$.
Thus, for any path satisfying (\ref{eq:arc_length}) and (\ref{eq:curv_constraint}), the y-axis coordinate is bounded above and below by the circular arc.
This implies that the segment $r_{\text{ball}}$ lies within the shaded region.

We next show that if $\kappa< 1/d_{\text{th}}$, then $r_{\text{ball}}$ cannot loop within the ball.
Note that, by definition, $r_{\text{ball}}$ loops within the ball if $x'(s) > 0$ for some point on the path within the shaded region.
The circular path with curvature $\kappa$ is the path that maximizes $x'(s)$.
From Fig. \ref{fig:d_th_loop_free_balls_and_ball_characterization}c, we can see that if $\kappa = 1/d_{\text{th}}$, then $x'(s)=0$ at $x(s)=-d_{\text{th}}$ for the circular path.
Thus, if $\kappa < 1/ d_{\text{th}}$, we have $x'(s)>0$ for any point of the path within the shaded region.

Finally, we determine the bound on the length of $r_{\text{ball}}$.
If we travel a distance of $d_{\text{th}}$ along the negative x-axis, then we are guaranteed to have exit the ball.
The path that maximizes its length before covering $d_{\text{th}}$ along the negative x-axis, would be the one that reduces the x-axis velocity $|x'(s)|$ the most.
This maximal length path corresponds to the circular path with constant curvature $\kappa$.
Any other path satisfying (\ref{eq:arc_length}) and (\ref{eq:curv_constraint}) would exit the shaded region before this circular path, i.e., the length of the segment of any path would be less that the length of this circular arc.
The length of this circular arc can be found from the geometry of the figure.
The chord length can be seen to be $2R_{c}\sin(\phi/2)$ where 
$R_c = 1/\kappa$.
Moreover, we have $\cos(\phi/2) = \frac{d_{\text{th}}}{2R_{c}\sin(\phi/2)}$ which implies that $\phi = \sin^{-1}\left(\frac{d_{\text{th}}}{R_c}\right)$.
This gives us the arc length as $2\pi R_c \times \frac{\phi}{2\pi} = R_c \sin^{-1}\left(\frac{d_{\text{th}}}{R_c}\right) = \frac{1}{\kappa}\sin^{-1}\left(\kappa d_{\text{th}}\right)$.
\end{proof}

\subsection{Proof of Lemma \ref{lemma:mean_std_KL_divergence}}\label{appendix:proof_KL_divergence_3_points}
\begin{proof}
Using (\ref{eq:est_mean}), we can show that $m = \alpha_1\Gamma_{\text{SH},-1} + \alpha_r\Gamma_{\text{SH},r}$ where 
\begin{align*}
\alpha_1 & = \frac{e^{-d_1/\beta_{\text{SH}}}-e^{-(d_{1r}+d_{r})/\beta_{\text{SH}}}}{1-e^{-2d_{1r}/\beta_{\text{SH}}}},
\end{align*}
\begin{align*}
\alpha_r & = \frac{e^{-d_r/\beta_{\text{SH}}}-e^{-(d_1+d_{1r})/\beta_{\text{SH}}}}{1-e^{-2d_{1r}/\beta_{\text{SH}}}}.
\end{align*}
Then, the difference in mean $\Delta m = m - \hat{m}$ is distributed as $\mathcal{N}(0, \sigma_{\Delta m}^2)$, where using (\ref{eq:del_m_var}) we have
\begin{align*}
\sigma_{\Delta m}^2 =  \sigma_{\text{SH}}^2 \frac{\left(e^{-d_r/\beta_{\text{SH}}}-e^{-(d_1+d_{1r})/\beta_{\text{SH}}}\right)^2}{1-e^{-2d_{1r}/\beta_{\text{SH}}}}.
\end{align*}
Moreover, using (\ref{eq:est_var}) we can calculate
\begin{align*}
\frac{\sigma^2}{\sigma^2_{\text{SH}}} & = 1 - \frac{e^{-2d_{1}/\beta_{\text{SH}}} + e^{-2d_{r}/\beta_{\text{SH}}}-2e^{-(d_{1}+d_{r}+d_{1r})/\beta_{\text{SH}}}}{1-e^{-2d_{1r}/\beta_{\text{SH}}}}.
\end{align*}
The difference in variance $\Delta \sigma^2 = \sigma^2 - \hat{\sigma}^2$ can be calculated as
\begin{align*}
\Delta \sigma^2 & = -\sigma_{\text{SH}}^2\frac{\left(e^{-d_r/\beta_{\text{SH}}}-e^{-(d_1+d_{1r})/\beta_{\text{SH}}}\right)^2}{1-e^{-2d_{1r}/\beta_{\text{SH}}}} \\
& = - \sigma_{\Delta m}^2.
\end{align*}
From (\ref{eq:KL_divergence}), we then have
\begin{align*}
KL & = \frac{\sigma_{\Delta m}^2}{2\hat{\sigma}^2}\chi_1^{2} + \frac{1}{2}\left(-\frac{|\Delta\sigma^2|}{\hat{\sigma}^2} - \log_{e}\left(1-\frac{|\Delta \sigma^2|}{\hat{\sigma}^2}\right)\right).
\end{align*}
Since $\mathbb{E}[\chi_{1}^{2}] = 1$ and $\mathrm{Var}[\chi_{1}^2] = 2$, we can calculate the mean $m_{KL}$ and the standard deviation $\sigma_{KL}$ to be as stated in the lemma.
\end{proof}

\subsection{Proof of Lemma \ref{lemma:ball_radius}}\label{appendix:proof_ball_radius}
\begin{proof}
Consider all possible locations of the general point (see Fig. \ref{fig:3_points_analysis_with_curvature}a) at a fixed distance $d_{r}$.
From the geometry of Fig. \ref{fig:3_points_analysis_with_curvature}a, we can see that $d_{1r} = \sqrt{d_1^2+d_r^2 -2d_1d_r\cos\theta}$.
Varying $\theta$, results in varying $d_{1r}$ which can take values in  $[d_{r}-d_1, d_r+d_1]$.
From Lemma \ref{lemma:mean_std_KL_divergence}, we can see that the $\theta$ that has a maximum impact on the KL divergence is the one that would minimize $m_{KL}$ and $\sigma_{KL}$.
This would occur when we maximize $\sigma_{\Delta m}^{2} = \sigma_{\text{SH}}^2 e^{-d_r/\beta_{\text{SH}}}\frac{(1-e^{-(z-z_{l})})^2}{1-e^{-2z}}$ where 
$z = d_{1r}/\beta_{\text{SH}}$ and $z_{l} = (d_r-d_1)/\beta_{\text{SH}}$.
We wish to maximize $h(z) = \frac{(1-e^{-(z-z_{l})})^2}{1-e^{-2z}}$.
Taking it's derivative gives us
\begin{align*}
\frac{\mathrm{d}}{\mathrm{d} z}h(z) = \frac{2(1-e^{-(z-z_l)})}{(1-e^{-2z})^2}(e^{-(z-z_l)}-e^{-2z}).
\end{align*}
Then $\frac{\mathrm{d}}{\mathrm{d} z}h(z)> 0$ if $z>-z_l$, which is true as long as $d_r > d_1$.

Thus, maximizing $\sigma_{\Delta m}^{2}$ occurs at $\theta = \pi$ where $d_{1r}$ takes its maximum value of $d_1+d_r$.
Setting $\theta = \pi$ gives us 
\begin{align*}
\sigma_{\Delta m}^{2} = \sigma_{\text{SH}}^2 \frac{\left(e^{-d_r/\beta_{\text{SH}}} - e^{-(2d_1 + d_r)/\beta_{\text{SH}}}\right)^2}{1 - e^{-2(d_1+d_r)/\beta_{\text{SH}}}}.
\end{align*}

From Lemma \ref{lemma:mean_std_KL_divergence}, we can see that satisfying the KL divergence parameters implies that 
$\frac{\sigma_{\Delta m}^2}{\hat{\sigma}^2} \leq 1-e^{-2\epsilon_{m}}$,
and 
$\frac{\sigma_{\Delta m}^2}{\hat{\sigma}^2} \leq \sqrt{2} \epsilon_{\sigma}$.
Let $\epsilon_d = \min\left\{1-e^{-2\epsilon_{m}}, \sqrt{2}\epsilon_{\sigma}\right\}$.
Thus, we obtain the constraint
\begin{align*}
\frac{e^{-2d_r/\beta_{\text{SH}}}(1-\rho^2)^2}{(1 - \rho^2e^{-2d_r/\beta_{\text{SH}}})(1-\rho^2)} \leq \epsilon_d,
\end{align*}
which in turn gives us the constraint 
\begin{align*}
d_{r} \geq \frac{\beta_{\text{SH}}}{2} \log_e\left(\rho^2 + \frac{1 - \rho^{2}}{\epsilon_d}\right).
\end{align*}

\end{proof}

\subsection{Proof of Lemma \ref{lemma:curv_constraint}}\label{appendix:proof_curv_constraint}
\begin{proof}
Consider the scenario of Fig. \ref{fig:3_points_analysis_with_curvature}b where 
$d_{1} = \Delta d$.
We will choose the location of the general point ($\Gamma_{\text{SH},r}$), which lies within the shaded region, such that it maximizes the impact (in terms of the KL divergence) on the approximation.
From Lemma \ref{lemma:mean_std_KL_divergence}, we can see that the point that has a maximum impact on the KL divergence is the one that would maximize $\sigma_{\Delta m}^{2}$.
From the proof of Lemma \ref{lemma:ball_radius}, we know that for a fixed $d_r$ and varying $\theta$, the maximum value of $\sigma_{\Delta m}^2$ occurs at the maximum value of $d_{1r}$.
This occurs at the boundary of the shaded region, i.e., at a point on the circular arc.
Since this holds for all $d_1<d_r\leq d_{\text{th}}$, we know that the point that maximizes $\sigma_{\Delta m}^{2}$ lies on the circular path with constant curvature $\kappa$.

We thus consider the setting in Fig. \ref{fig:3_points_analysis_with_curvature}c with a fixed curvature $\kappa$.
From the geometry of the figure, we have the following relations:
$d_1 = 2R_c \sin\left(\frac{\Delta \phi}{2}\right)$, $d_{1r} = 2R_c \sin\left(\frac{\phi}{2}\right)$ and $d_{r} = 2R_c \sin\left(\frac{\phi + \Delta \phi}{2}\right)$.
Since $d_1 = \Delta d$, we have $\Delta \phi = 2\sin^{-1}(\kappa \Delta d/2)$.
From Lemma \ref{lemma:dth_loop_free}, we have the constraint that $\kappa < 1/d_{\text{th}}$.
This guarantees that the path will leave the ball.
Moreover, from the geometry of the figure, we can see that this will occur at the angle $\phi$ such that $d_{r} = 2R_c \sin\left(\frac{\phi + \Delta \phi}{2}\right) = d_{\text{th}}$.
This occurs at $\phi = h_{\text{cons}}(\kappa) =  2\sin^{-1}(\frac{\kappa d_{\text{th}}}{2})-\Delta \phi$.

From Lemma \ref{lemma:mean_std_KL_divergence}, we can see that satisfying the KL divergence parameters implies that 
$\frac{\sigma_{\Delta m}^2}{\hat{\sigma}^2} \leq 1-e^{-2\epsilon_{m}}$,
and 
$\frac{\sigma_{\Delta m}^2}{\hat{\sigma}^2} \leq \sqrt{2} \epsilon_{\sigma}$.
Let $\epsilon_d = \min\left\{1-e^{-2\epsilon_{m}}, \sqrt{2}\epsilon_{\sigma}\right\}$.
Thus, the point on the path that maximizes the KL divergence occurs at the angle 
\begin{align*}
\arg\max_{0<\phi\leq h_{\text{cons}}(\kappa)} h_{\text{opt}}(\kappa, \phi),
\end{align*}
where 
\begin{align*}
h_{\text{opt}}(\kappa, \phi) & = \frac{\sigma_{\Delta m}^2}{\hat{\sigma}^2}\\
& = \frac{\left( e^{-\frac{2}{\kappa\beta_{\text{SH}}}\sin(\frac{\phi + \Delta\phi}{2})} -  \rho e^{-\frac{2}{\kappa\beta_{\text{SH}}}\sin(\frac{\phi}{2})} \right)^2}
{(1 - e^{-\frac{4}{\kappa\beta_{\text{SH}}}\sin(\frac{\phi}{2})})(1-\rho^2)}.
\end{align*}
We wish to find the maximum curvature $\kappa$, such that this maximum impact still satisfies the KL divergence parameters, i.e., 
\begin{align*}
\max_{0<\phi\leq h_{\text{cons}}(\kappa)} h_{\text{opt}}(\kappa, \phi) \leq \epsilon_d.
\end{align*}
This results in the optimization problem stated in the lemma.
\end{proof}

\bibliographystyle{spbasic}      
\bibliography{ref}{}

\begin{thebibliography}{31}
\providecommand{\natexlab}[1]{#1}
\providecommand{\url}[1]{{#1}}
\providecommand{\urlprefix}{URL }
\expandafter\ifx\csname urlstyle\endcsname\relax
  \providecommand{\doi}[1]{DOI~\discretionary{}{}{}#1}\else
  \providecommand{\doi}{DOI~\discretionary{}{}{}\begingroup
  \urlstyle{rm}\Url}\fi
\providecommand{\eprint}[2][]{\url{#2}}

\bibitem[{Caccamo et~al.(2017)Caccamo, Parasuraman, Freda, Gianni, and
  Ogren}]{sergio2017RCAMP}
Caccamo S, Parasuraman R, Freda L, Gianni M, Ogren P (2017) Rcamp: Resilient
  communication-aware motion planner and autonomous repair of wireless
  connectivity in mobile robots. In: IEEE/RSJ International Conference on
  Intelligent Robots and Systems (IROS), pp 2153--0866

\bibitem[{Chatzipanagiotis and
  Zavlanos(2016)}]{chatzipanagiotis2016distributed}
Chatzipanagiotis N, Zavlanos MM (2016) Distributed scheduling of network
  connectivity using mobile access point robots. IEEE Transactions on Robotics
  32(6):1333--1346

\bibitem[{Cover and Thomas(2012)}]{cover2012elements}
Cover TM, Thomas JA (2012) Elements of information theory. John Wiley \& Sons

\bibitem[{Craig(2008)}]{craig2008new}
Craig P (2008) A new reconstruction of multivariate normal orthant
  probabilities. Journal of the Royal Statistical Society: Series B
  (Statistical Methodology) 70(1):227--243

\bibitem[{Di~Nardo et~al.(2001)Di~Nardo, Nobile, Pirozzi, and
  Ricciardi}]{di2001computational}
Di~Nardo E, Nobile A, Pirozzi E, Ricciardi L (2001) A computational approach to
  first-passage-time problems for gauss--markov processes. Advances in Applied
  Probability 33(2):453--482

\bibitem[{Doob(1949)}]{doob1949heuristic}
Doob JL (1949) Heuristic approach to the kolmogorov-smirnov theorems. The
  Annals of Mathematical Statistics 20(3):393--403

\bibitem[{Dudley(2002)}]{dudley2002real}
Dudley RM (2002) Real analysis and probability, vol~74. Cambridge University
  Press

\bibitem[{Eberly(2008)}]{eberly2008moving}
Eberly D (2008) Moving along a curve with specified speed. Preprint, see
  http://www geometrictools com p~2

\bibitem[{Gardiner(2009)}]{gardiner2009stochastic}
Gardiner C (2009) Stochastic methods. Springer Berlin

\bibitem[{Hashemi(1994)}]{hashemi1994study}
Hashemi H (1994) A study of temporal and spatial variations of the indoor radio
  propagation channel. In: IEEE International Symposium on Personal, Indoor and
  Mobile Radio Communications, IEEE, pp 127--134

\bibitem[{Kay(1993)}]{kay1993fundamentals}
Kay SM (1993) Fundamentals of statistical signal processing, volume i:
  Estimation theory (v. 1). PTR Prentice-Hall, Englewood Cliffs

\bibitem[{Kline(1998)}]{kline1998calculus}
Kline M (1998) Calculus: an intuitive and physical approach. Courier
  Corporation

\bibitem[{Lancaster and Seneta(2005)}]{lancaster2005chi}
Lancaster HO, Seneta E (2005) Chi-square distribution. Encyclopedia of
  biostatistics 2

\bibitem[{Leblanc and Scaillet(1998)}]{leblanc1998path}
Leblanc B, Scaillet O (1998) Path dependent options on yields in the affine
  term structure model. Finance and Stochastics 2(4):349--367

\bibitem[{Malmirchegini and Mostofi(2012)}]{malmirchegini2012spatial}
Malmirchegini M, Mostofi Y (2012) On the spatial predictability of
  communication channels. IEEE Transactions on Wireless Communications
  11(3):964--978

\bibitem[{Mehr and McFadden(1965)}]{mehr1965certain}
Mehr C, McFadden J (1965) Certain properties of gaussian processes and their
  first-passage times. Journal of the Royal Statistical Society Series B
  (Methodological) pp 505--522

\bibitem[{Muralidharan and
  Mostofi(2017{\natexlab{a}})}]{muralidharan2017energy}
Muralidharan A, Mostofi Y (2017{\natexlab{a}}) Energy optimal distributed
  beamforming using unmanned vehicles. IEEE Transactions on Control of Network
  Systems

\bibitem[{Muralidharan and Mostofi(2017{\natexlab{b}})}]{muralidharan2017first}
Muralidharan A, Mostofi Y (2017{\natexlab{b}}) First passage distance to
  connectivity for mobile robots. In: American Control Conference (ACC), IEEE,
  pp 1517--1523

\bibitem[{Muralidharan and Mostofi(2017{\natexlab{c}})}]{muralidharan2017path}
Muralidharan A, Mostofi Y (2017{\natexlab{c}}) Path planning for a connectivity
  seeking robot. In: Globecom Workshops (GC Wkshps), IEEE, pp 1--6

\bibitem[{Papoulis and Pillai(2002)}]{papoulis2002probability}
Papoulis A, Pillai SU (2002) Probability, random variables, and stochastic
  processes. Tata McGraw-Hill Education

\bibitem[{Rappaport(1996)}]{rappaport1996wireless}
Rappaport TS (1996) Wireless communications: principles and practice, vol~2.
  prentice hall PTR New Jersey

\bibitem[{Ricciardi and Sacerdote(1979)}]{ricciardi1979ornstein}
Ricciardi LM, Sacerdote L (1979) The ornstein-uhlenbeck process as a model for
  neuronal activity. Biological cybernetics 35(1):1--9

\bibitem[{Ricciardi and Sato(1988)}]{ricciardi1988first}
Ricciardi LM, Sato S (1988) First-passage-time density and moments of the
  ornstein-uhlenbeck process. Journal of Applied Probability pp 43--57

\bibitem[{Robert(1996)}]{robert1996intrinsic}
Robert CP (1996) Intrinsic losses. Theory and decision 40(2):191--214

\bibitem[{Siegert(1951)}]{siegert1951first}
Siegert AJ (1951) On the first passage time probability problem. Physical
  Review 81(4):617

\bibitem[{Smith and Cox(2004)}]{smith2004urban}
Smith WM, Cox DC (2004) Urban propagation modeling for wireless systems. Tech.
  rep., DTIC Document

\bibitem[{Tokekar et~al.(2016)Tokekar, Vander~Hook, Mulla, and
  Isler}]{tokekar2016sensor}
Tokekar P, Vander~Hook J, Mulla D, Isler V (2016) Sensor planning for a
  symbiotic uav and ugv system for precision agriculture. IEEE Transactions on
  Robotics 32(6):1498--1511

\bibitem[{Yan and Mostofi(2012)}]{yan2012robotic}
Yan Y, Mostofi Y (2012) Robotic router formation in realistic communication
  environments. IEEE Transactions on Robotics 28(4):810--827

\bibitem[{Yan and Mostofi(2014)}]{YanMostofiTCSN14}
Yan Y, Mostofi Y (2014) {To Go or Not to Go On Energy-Aware and
  Communication-Aware Robotic Operation}. IEEE Transactions on Control of
  Network Systems 1(3):218 -- 231

\bibitem[{Zavlanos et~al.(2011)Zavlanos, Egerstedt, and
  Pappas}]{zavlanos2011graph}
Zavlanos MM, Egerstedt MB, Pappas GJ (2011) Graph-theoretic connectivity
  control of mobile robot networks. Proceedings of the IEEE 99(9):1525--1540

\bibitem[{Zeng and Zhang(2017)}]{zeng2017energy}
Zeng Y, Zhang R (2017) Energy-efficient uav communication with trajectory
  optimization. IEEE Transactions on Wireless Communications 16(6):3747--3760

\end{thebibliography}

\end{document}